%% file: tensorcomp.tex
\documentclass{article} 
\usepackage[parfill]{parskip} 
\usepackage{fullpage, prettyref}
\usepackage{hyperref}

\usepackage{amssymb,amsfonts,amsmath,amsthm,amscd,dsfont,mathrsfs}
\usepackage{graphicx,float,psfrag,epsfig}
\usepackage{wrapfig}
\usepackage{algorithm,algorithmic}
\DeclareMathAlphabet{\mathpzc}{OT1}{pzc}{m}{it}

\newtheorem{propo}{Proposition}[section]
\newtheorem{lemma}[propo]{Lemma}

\newtheorem{coro}[propo]{Corollary}
\newtheorem{thm}[propo]{Theorem}
\newtheorem{theorem}[propo]{Theorem}

\def\cP{{\cal P}}
\def\cPO{{\cal P}_\Omega}

\def\cL{{\cal L}}
\def\bcL{\overline{\cal L}}

\def\reals{{\mathbb R}}
\def\int{{\mathbb Z}}
\def\prob{{\mathbb P}}
\def\E{\mathbb E}
\def\deg{{\rm deg}}

\def\ones{{\mathds 1}}
\def\id{{\mathds I}}

\def\be{{\bar e}}

\def\r{{r}}

\def\a{{a}}

\def\<{\langle}
\def\>{\rangle}

\def\tS{\widetilde S}

\def\cP{{\cal P}}

\def\hT{\widehat T}

\def\v{v}

\def\poly{{\rm poly}}
\def\tepsilon{{\tilde{\varepsilon}}}

\def\R{\reals}
\def\u{{\bf u}}
\def\v{{\bf v}}
\def\e{{\bf e}}
\def\uo{\u^*}
\def\vo{\v^*}

\def\uoi{u^*_i}
\def\uoj{u^*_j}

\def\ui{u_i}
\def\uj{u_j}
\def\uk{u_k}

\def\err{{\rm err}}
\def\eps{\epsilon}
\def\Tmax{T_{\rm max}}

\newcommand{\ip}[2]{\langle #1,#2\rangle}

\def\poly{{\rm poly}}

\def\uoi{\uo_i}
\def\uoj{\uo_j}

\def\uol{\uo_\ell}
\def\uoq{\uo_q}

\def\ui{\u_i}
\def\uj{\u_j}
\def\uk{\u_k}
\def\uq{\u_q}

\def\ul{\u_\ell}
\def\dv{{\bf d}^v}
\def\du{{\bf d}^u}

\def\dul{{\bf d}_\ell}

\def\dulj{\dul(j)}
\def\dulk{\dul(k)}
\def\duq{{\bf d}_q}

\def\uqj{\uq(j)}
\def\uqk{\uq(k)}

\def\uoqi{\uoq(i)}
\def\uoqj{\uoq(j)}
\def\uoqk{\uoq(k)}
\def\ulj{\ul(j)}
\def\ulk{\ul(k)}
\def\uli{\ul(i)}
\def\uoli{\uol(i)}
\def\uolj{\uol(j)}
\def\uolk{\uol(k)}
\def\sq{\sigma_q}
\def\sll{\sigma_\ell}
\def\soq{\sigma_q^*}
\def\sol{\sigma_\ell^*}
\def\sdl{\Delta^\sigma_\ell}
\def\sdq{\Delta^\sigma_q}

\begin{document}
\date{}
\title{Provable Tensor Factorization with Missing Data}

\author{
Prateek Jain\thanks{Microsoft Research, \texttt{prajain@microsoft.com}} 
 and Sewoong Oh\thanks{
 Department of Industrial and Enterprise Systems Engineering, 
 Univrsity of Illinois at Urbana-Champaign, \texttt{swoh@illinois.edu}}
}

\maketitle

\begin{abstract}
We study the problem of low-rank tensor factorization in the presence of  
missing data. 
We ask the following question: how many sampled entries do we need, to 
efficiently and exactly reconstruct a  tensor 
with a low-rank orthogonal decomposition?  
We propose a novel alternating minimization based method 
which iteratively refines estimates of the singular vectors. 
We show that under certain standard assumptions, 
our method can recover a three-mode $n\times n\times n$ dimensional rank-$r$  tensor exactly 
from $O(n^{3/2} r^5 \log^4 n)$ randomly sampled entries. 
In the process of proving this result, we 
solve two challenging sub-problems for tensors with missing data. 
First, in the process of analyzing the initialization step, 
we  prove a generalization of a celebrated result by Szemer\'edie et al. on the spectrum of random graphs. 
Next,  we prove global convergence of alternating minimization with a good initialization. 
Simulations suggest that the dependence of the sample size on dimensionality $n$ is indeed tight. 
\end{abstract}

\input{intro}
\input{results}

\input{discussion}
\bibliographystyle{unsrt}
\bibliography{mturk}
\input{appendix}
\end{document}

%% file: intro.tex
\section{Introduction}
\label{sec:intro} 
Several real-world applications routinely encounter multi-way data with structure which can be modeled as low-rank tensors. Moreover, in several settings, many of the entries of the tensor are missing, which motivated us to study the problem of low-rank tensor factorization with missing entries. 
For example, when recording electrical activities of the brain, 
the electroencephalography (EEG) signal can be represented as a 
three-way array  (temporal, spectral, and spatial axis). 
Oftentimes signals are lost due to mechanical failure or loose connection. 
Given numerous motivating applications, several methods have been proposed for this tensor completion problem. 
However, with the exception of 2-way tensors (i.e., matrices), the existing methods for higher-order tensors do not have theoretical guarantees and typically suffer from the curse of local minima. 

In general, finding a factorization of a tensor is an NP-hard problem, even when all the entries are available. 
However, it was recently discovered that by restricting  attention to a sub-class of tensors such as 
low-CP rank orthogonal tensors \cite{AGHKT12} or low-CP rank incoherent\footnote{The notion of incoherence we assume in \eqref{eq:defmu} can be thought of as incoherence between the  fibers and the standard basis vectors.} tensors \cite{AGJ14}, one can efficiently find a provably approximate factorization. In particular, 
exact recovery of the factorization is possible for a tensor with a low-rank orthogonal CP decomposition \cite{AGHKT12}. 
We ask the question of recovering such a CP-decomposition when only a small number of entries are revealed, 
and show that exact reconstruction is possible even when we do not observe any entry in most of the fibers. 

\noindent{\bf Problem formulation.}
We study tensors that have an orthonormal CANDECOMP/PARAFAC (CP) tensor decomposition with a small number of components. Moreover, for simplicity of notation and exposition, we only consider symmetric third order tensors. We would like to stress that our techniques generalizes easily to handle {\em non-symmetric tensors} as well as {\em higher-order} tensors.
Formally, we assume that the true tensor $T$ has the the following form:
\begin{eqnarray}
	T &=& \sum_{\ell=1}^\r \sigma_\ell (\u_\ell \otimes \u_\ell \otimes \u_\ellㅣ) \;\in\; \reals^{n \times n \times n}\;,\label{eq:defcp}
\end{eqnarray}
with $\r \ll n$, $u_\ell\in\reals^n$ with $\|u_\ell\|=1$, and $u_\ell$'s are orthogonal to each other.  
We let $U\in \reals^{n\times r}$ be a tall-orthogonal matrix where $u_\ell$'s is the $\ell$-th column of $U$ 
and $U_i \perp U_j$ for $i\neq j$.  We use  $\otimes$ to denote the standard outer product such that 
the $(i,j,k)$-th element of $T$ is given by:  $T_{ijk}=\sum_a \sigma_a U_{ia} U_{ja}  U_{ka}$. 
We further assume that the $u_i$'s are unstructured, which is formalized by the notion of {\em incoherence} 
commonly assumed in matrix completion problems. 
The {\em incoherence} of a symmetric tensor with orthogonal decomposition is 
\begin{eqnarray}
	\mu(T) &\equiv&  \max_{i\in[n],\ell\in[r]}\;\sqrt{n}\, |U_{i\ell}| \;,
\label{eq:defmu}
\end{eqnarray}
where $[n]=\{1,\ldots, n\}$ is the set of the first $n$ integers. 
Tensor completion becomes increasingly difficult for tensors with larger $\mu(T)$, because the `mass' of the tensor can be concentrated on 
a few entries that might not be revealed. 
Out of $n^3$ entries of $T$, 
a subset $\Omega\subseteq [n]\times [n] \times [n]$ is revealed. 
We use $\cP_\Omega(\cdot)$ to denote 
the projection of a matrix onto the revealed set such that 
\begin{eqnarray*}
	\cP_\Omega(T)_{ijk} &=& \left\{ 
	\begin{array}{rl}
		T_{ijk} & \text{ if }(i,j,k)\in\Omega \;,\\
		0 & \text{ otherwise }\;. 
	\end{array}
	\right.
\end{eqnarray*}
We want to recover $T$ exactly using the given entries ($P_\Omega(T)$). 
We assume that each $(i,j,k)$ for all $i\leq j\leq k$ is included in $\Omega$ with a {\em fixed probability $p$} (since $T$ is symmetric, we include all permutations of $(i,j,k)$). 
This is equivalent to fixing the total number of samples $|\Omega|$ and 
selecting $\Omega$ {\em uniformly at random} over all ${n^3 \choose |\Omega|}$ choices. 
The goal is to ensure exact recovery with high probability and for 
$|\Omega|$ that is sub-linear in the number of entries ($n^3$). 

\noindent{\bf Notations.} 
For a tensor $T\in\reals^{n\times n\times n}$, we define a linear mapping using $U\in\reals^{n\times m}$ as 
$T[U,U,U] \in\reals^{m\times m\times m}$ such that $T[U,U,U]_{ijk} = \sum_{a,b,c} T_{abc}U_{ai}U_{bj}U_{ck}$.
The operator norm of a tensor is $\|T\|_2=\max_{\|x\|=1} T[x,x,x]$. 
The Frobenius norm of a tensor is $\|T\|_F = (\sum_{i,j,k}T_{ijk}^2)^{1/2}$.
The Euclidean norm of a vector is $\|\u\|_2 = (\sum_i \u_i^2)^{1/2}$.
We use $C,C'$ to denote any positive numerical constants and the actual value might change from line to line. 

\subsection{Algorithm} 
\label{sec:algo}

Ideally, one would like to minimize the rank of a tensor that explains all the sampled entries. 
	\begin{eqnarray}
		\label{eq:prob}
		\underset{\hT}{\text{minimize}} && {\rm rank}(\hT)\\
		\text{subject to} && T_{ijk}=\hT_{ijk} \text{ for all } (i,j,k)\in\Omega\;.\nonumber
	\end{eqnarray}
However, even computing the rank of a tensor is NP-hard in general, 
where the rank is defined as the minimum $r$ for which CP-decomposition exists \cite{SL08}. 
Instead, we fix the rank of $\hT$ by explicitly modeling $\hT$ as $\hT=\sum_{\ell\in[r]}\sll ( \ul\otimes \ul\otimes \ul)$, and solve the following problem:
\begin{eqnarray}\label{eq:prob_eq}
\underset{\hT, \text{rank}(\hT)=r }{\text{minimize}}\Big\|\cP_\Omega(T)- \cP_\Omega\big(\hT\big)\Big\|_F^2&=&\underset{\{\sigma_\ell,\ul\}_{\ell\in[r]} }{\text{minimize}} \Big\|\cP_\Omega(T)- \cP_\Omega\big(\sum_{\ell\in[r]}\sll ( \ul\otimes \ul\otimes \ul)\,\big)\Big\|_F^2
\end{eqnarray}
Recently, \cite{JNS13, hardt2013provable} showed that an alternating minimization technique, 
can  recover a matrix with missing entries exactly. 
We generalize and modify the algorithm for the case of higher order tensors and 
study it rigorously for tensor completion. However, due to special structure in higher-order tensors, our algorithm as well as analysis is significantly different than the matrix case (see Section~\ref{sec:resultam} for more details). 

To perform the minimization, we repeat the outer-loop getting refined estimates for all $r$ components. 
In the inner-loop, we loop over each component and solve for $\uq$ while fixing the others $\{\ul\}_{\ell\neq q}$. More precisely, we set $\hT=\uq^{t+1}\otimes \uq \otimes \uq + \sum_{\ell\neq q}\sigma_\ell\ul\otimes \ul\otimes \ul$ in \eqref{eq:prob_eq} and then find optimal $\uq^{t+1}$ by minimizing the least squares objective given by \eqref{eq:prob_eq}. 
That is, each inner iteration is a simple least squares problem over the known entries, 
hence can be implemented efficiently and is also embarrassingly parallel. 

\begin{algorithm}[ht]
\caption{Alternating Minimization for Tensor Completion}
\label{algo:rkk}
  \begin{algorithmic}[1]
    \STATE Input: $P_\Omega(T)$, $\Omega$, $\r$, $\tau$
    \STATE Initialize with $[(\u_1^0,\sigma_1), (\u_2^0,,\sigma_2), \dots, (\u_r^0,\sigma_r)]={\sc RTPM}(P_\Omega(T), r)$ \hfill (RTPM of \cite{AGHKT12})
    \STATE $[\u_1, \u_2, \dots, \u_r]={\rm Threshold}([\u_1^0, \u_2^0, \dots, \u_r^0])$ \hfill (Clipping scheme of \cite{JNS13})
    \FORALL {$t=1, 2, \dots, \tau$} 
    \STATE /*OUTER LOOP */
    \FORALL {$q=1, 2, \dots, r$}
    \STATE /*INNER LOOP*/
    \STATE $\hat{\u}_1^{t+1}= \arg\min_{\uq^{t+1}} \| \cP_\Omega(T- \uq^{t+1}\otimes \uq\otimes \uq-\sum_{\ell\neq q}\sll\cdot \ul\otimes \ul\otimes \ul)\|_F^2$
    \STATE $\sq^{t+1}=\|\hat{\uq}^{t+1}\|_2$
    \STATE $\uq^{t+1}=\hat{\u}_1^{t+1}/\|\hat{\u}_q^{t+1}\|_2$
    \ENDFOR 
    \STATE $[\u_1, \u_2, \dots, \u_r]\gets [\u_1^{t+1}, \u_2^{t+1}, \dots, \u_r^{t+1}]$
    \STATE $[\sigma_1, \sigma_2, \dots, \sigma_r]\gets [\sigma_1^{t+1}, \sigma_2^{t+1}, \dots, \sigma_r^{t+1}]$
    \ENDFOR
    \STATE Output: $\hT = \sum_{q\in[\r]}  \sq (\uq \otimes \uq \otimes \uq)$
  \end{algorithmic}
\end{algorithm}

The {\em main novelty} in our approach is that we refine all $r$ components iteratively as opposed to the sequential deflation technique used by the existing methods for tensor decomposition (for fully observed tensors). In sequential deflation methods, components $\{\u_1, \u_2, \dots, \u_r\}$ are estimated sequentially and estimate of say $\u_2$ is not used to refine  $\u_1$. In contrast, our algorithm  iterates over all $r$ estimates in the inner loop, so as to obtain refined estimates for all $\ui$'s in the outer loop. 
We believe that such a technique could be applied to improve the error bounds of (fully observed) tensor decomposition methods as well. 

As our method is directly solving a non-convex problem, it can easily get stuck in local minima. The key reason our approach can overcome the curse of local minima is that we start with a provably good initial point which is only a small distance away from the optima. To obtain such an initial estimate, 
we compute a low-rank approximation of the observed tensor using {\em Robust Tensor Power Method} (RTPM) \cite{AGHKT12}. 
RTPM is a generalization of the widely used 
power method for computing leading singular vectors of a matrix and can approximate the largest singular vectors up to the spectral norm of the ``error'' tensor. 
Hence, the challenge is to show that the error tensor has small spectral norm (see Theorem \ref{thm:tcinit}). We perform a thresholding step similar to \cite{JNS13} (see Lemma~\ref{lem:threshold}) after the RTPM step to ensure that the estimates we get are incoherent, which is critical for our analysis. 

Our analysis requires the sampled entries $\Omega$ to be independent of the current iterates $\ui, \forall i$, 
which in general is not possible as $\ui$'s are computed using $\Omega$. To avoid this issue, we divide the given samples 
($\Omega$) into equal $r\cdot \tau$ parts randomly where $\tau$ is the number of outer loops (see Algorithm~\ref{algo:rkk}). 

\subsection{\bf Main Result}

\begin{thm}
 \label{thm:main}
	Consider any rank-$r$ symmetric tensor $T\in\reals^{n\times n\times n}$ with an orthogonal CP decomposition in \eqref{eq:defcp} 
	satisfying $\mu$-incoherence as defined in \eqref{eq:defmu}. 
	For any positive $\varepsilon>0$, there exists a positive numerical constant $C$ such that 
	if entries are revealed with probability 
	\begin{eqnarray*}
		p &\geq& C\, \frac{\mu^6\, r^5\, \sigma_{\rm max}^4\,(\log n)^4\,\log(r\|T\|_F/\varepsilon)}{\sigma_{\rm min}^4\,n^{3/2}} \;,
	\end{eqnarray*}
	then the following holds with probability at least $1-n^{-5}\log_2(4\sqrt{r}\,\|T\|_F/\varepsilon)$: 
	\begin{itemize}
	 \item the problem \eqref{eq:prob} has a unique optimal solution; and 
	 \item 	$\log_2(\frac{4\sqrt{r}\,\|T\|_F}{\varepsilon})$ iterations of Algorithm \ref{algo:rkk} 
	produces an estimate $\hT$ s.t. 	$\|T-\hT\|_F \leq \varepsilon\;.$  
	\end{itemize}
\end{thm}
Note that the above result can be generalized to $k$-mode tensors in a straightforward manner, where exact recovery is guaranteed if,  $p \geq C\, \frac{\mu^6\, r^5\, \sigma_{\rm max}^4\,(\log n)^4\,\log(r\|T\|_F/\varepsilon)}{\sigma_{\rm min}^4\,n^{k/2}}$. However, for simplicity of notation and to emphasize key points of our proof we present our proof for $3$-mode tensors only in Section~\ref{sec:resultmain}.

We provide a proof of Theorem~\ref{thm:main} in Section \ref{sec:results}. 
For an incoherent, well-conditioned, and low-rank tensor with $\mu=O(1)$ and $\sigma_{\rm min}=\Theta(\sigma_{\rm max})$, 
alternating minimization requires $O(r^5 n^{3/2}(\log n)^4)$ samples to get within an arbitrarily small normalized error. 
This is a vanishing fraction of the total number of entries $n^3$. 
Each step in the alternating minimization requires $O(r |\Omega|)$ operations, 
hence the alternating minimization only requires $O(r |\Omega| \log(r\|T\|_F/\varepsilon) )$ operations.
The initialization step requires $O(r^c|\Omega|)$ operations for some positive constant numerical $c$. 
When $r\ll n$, the computational complexity scales linearly in the sample size up to a logarithmic factor. 

A fiber in a third order tensor is an 
$n$-dimensional vector defined by fixing two of the axes and indexing over remaining one axis. 
The above theorem implies that among $n^2$ fibers of the form $\{T[\id,e_j,e_k]\}_{j,k\in[n]}$, 
it is sufficient to have only $O(n^{3/2}(\log n)^4)$ fibers with {\em any} samples. 
Most of the fibers are not sampled at all and, perhaps surprisingly, our approach can still recover the original low-rank tensor. 
This should be compared to the matrix completion setting where all fibers are required to have at least one sample. 

However, unlike matrices, the fundamental limit of higher order tensor completion is not known. 
Building on the percolation of Erd\"os-Ren\'yi graphs and the coupon-collectors problem, 
it is known that matrix completion has multiple rank-$r$ solutions 
when the sample size is less than $C\mu r n \log n$ \cite{CT10}, hence exact recovery is impossible.
But, such arguments do not generalize directly to higher order; see Section~\ref{sec:discussion} for more discussion. Interestingly, simulations in Section \ref{sec:exp} suggests that for $r=O(\sqrt{n})$, 
the sample complexity scales as $(r^{1/2}n^{3/2}\log n)$. That is, assuming the sample complexity provided by simulations is correct, our result achieves optimal dependence on $n$ (up to $\log$ factors). However, the dependency on $r$ is sub-optimal (see Section \ref{sec:discussion} for a discussion). 

\input{exps}

\subsection{Related Work}
{\em Tensor decomposition and completion}: 
The CP model proposed in \cite{Hit27,CC70,Har70} is a multidimensional generalization of singular value decomposition of matrices. 
Computing the CP decomposition involves two steps: 
first apply a whitening operator to the tensor to get 
a lower dimensional tensor with orthogonal CP decomposition. 
Such a whitening operator only exists when $r\leq n$.
Then, apply known power-method techniques for exact orthogonal CP decomposition \cite{AGHKT12}. 
We use this algorithm as well as the analysis 
for the initial step of our algorithm. 
For motivation and examples of orthogonal CP models we refer to \cite{ZG01,AGHKT12}. 

Recently, many heuristics for tensor completion have been developed such as 
the weighted least squares \cite{ADKM11}, Gauss-Newton \cite{TB05}, 
alternating least-squares \cite{Bro98,WM01}, trace norm minimization \cite{LMP13}.
However, to the best of our knowledge, there is no tensor completion method with provable guarantees. 
In a different context, \cite{MHWG13} show that minimizing a weighted trace norm of flattened tensor provides exact recovery using $O(rn^{3/2})$ samples, but each observation needs to be a dense random projection of the tensor as opposed to observing just a single entry, which is the case in the tensor completion problem.

{\em Relation to matrix completion}: 
Matrix completion has been studied extensively in the last decade since 
the seminal paper by Candes and Recht \cite{CR09}. Since then, several provable approaches have been developed, such as, nuclear norm minimization \cite{CR09}, OptSpace \cite{KMO10}, and Alternating Minimization \cite{JNS13}. However, several aspects of tensor factorization makes 
it challenging to adopt matrix completion algorithms and 
analysis techniques directly. 

First, there is no natural convex surrogate of the tensor rank function and developing such a function is in fact a topic of active research \cite{TomiokaS13,MHWG13}
Next, even when all entries are revealed, tensor decomposition methods 
such as simultaneous power iteration are known to get stuck at local extrema, 
making it challenging to 
apply matrix decomposition methods directly. 
Third, for the initialization step, 
the best low-rank approximation of a matrix is unique and finding it is trivial. 
However, for tensors, finding the best low-rank approximation is notoriously difficult. 

On the other hand, some aspects of tensor decomposition makes it possible to prove stronger results. 
Matrix completion aims to recover 
the underlying matrix only, since the factors are not uniquely defined due to invariance under 
rotations. 
However, for orthogonal CP models, we can hope to recover the 
individual singular vectors $\u_i$'s exactly. In fact, Theorem~\ref{thm:main} shows that our method indeed recovers the individual singular vectors exactly. 
{\em Spectral analysis of tensors and hypergraphs}: 
Theorem \ref{thm:tcinit} and Lemma \ref{lem:tcinit} should be compared to copious line of work on spectral analysis of matrices \cite{AFK01,KMO10}, with an important motivation of developing fast algorithms for low-rank matrix approximations. 
We prove an analogous guarantee for higher order tensors and provide a fast algorithm for low-rank tensor approximation.  Theorem \ref{thm:tcinit} 
is also a generalization of the celebrated result of Friedman-Kahn-Szemer\'edi \cite{FKS89} and 
Feige-Ofek \cite{FeO05} on the second eigenvalue of random graphs.  
We provide an upper bound the largest second eigenvalue of a random hypergraph, 
where each edge includes three nodes and each of the ${n \choose 3}$ edges is selected with probability $p$.




%% file: exps.tex
\subsection{Empirical Results}
\label{sec:exp}

Theorem \ref{thm:main} guarantees exact recovery when $p\geq C r^5 (\log n)^4/ n^{3/2}$.
Numerical experiments 
show that the average recovery rate converges to a universal curve over $\alpha$, where 
$p^* = \alpha r^{1/2}\ln n /((1-\rho) n^{3/2})$ in Figure \ref{fig:threshold}. 
Our bound is tight in its dependency $n$ up to a poly-logarithmic factor, 
but is loose in its dependency in the rank $r$. 
Further, it is able to recover the original matrix exactly even when the factors are not strictly orthogonal. 

We generate orthogonal matrices $U=[\u_1,\ldots,\u_r]\in\reals^{n\times r}$ uniformly at random 
with $n=50$ and $r=3$ unless specified otherwise. 
For a rank-$r$ tensor $T=\sum_{i=1}^r \u_i\otimes \u_i\otimes \u_i$, we randomly reveal each entry with probability $p$. 
A tensor is exactly recovered if the normalized root mean squared error, ${\rm RMSE}=\|T-\hat{T}\|_F/\|T\|_F$, is less than $10^{-7}$. 
Varying $n$ and $r$, we plot the recovery rate averaged over $100$ instances as a function of $\alpha$. 
The degrees of freedom in representing a symmetric tensor is $rn-r^2$. 
Hence for large, $r$ we need number of samples scaling as $r$.
Hence, the current dependence of $p^*=O(\sqrt{r})$ can only hold for $r=O(n)$. 
For not strictly orthogonal factors, the algorithm is robust. 
A more robust approach for finding an initial guess could improve the performance significantly, 
especially for non-orthogonal tensors.

\begin{figure}[h]
 \begin{center}
	\includegraphics[width=.3\textwidth]{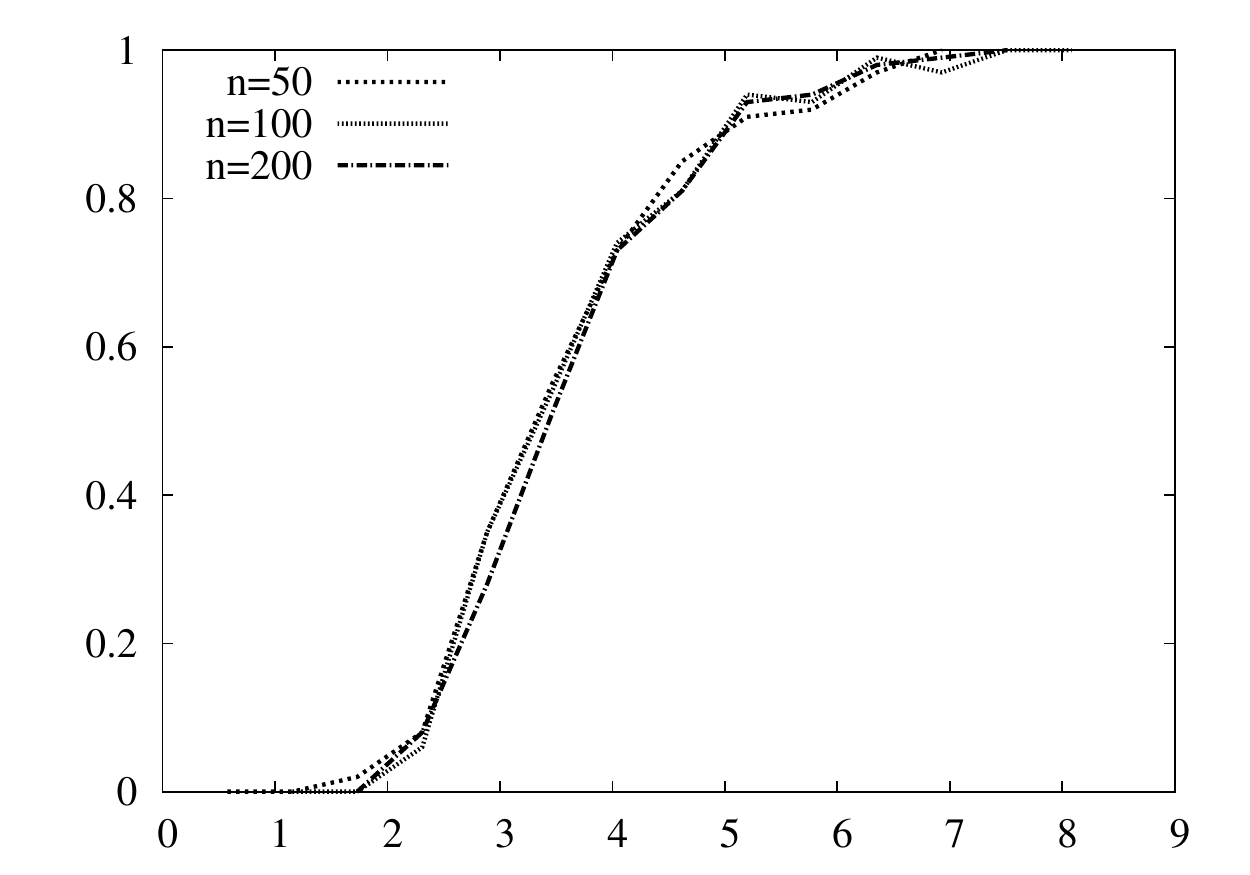}
	\put(-60,-5){$\alpha$}
	\includegraphics[width=.3\textwidth]{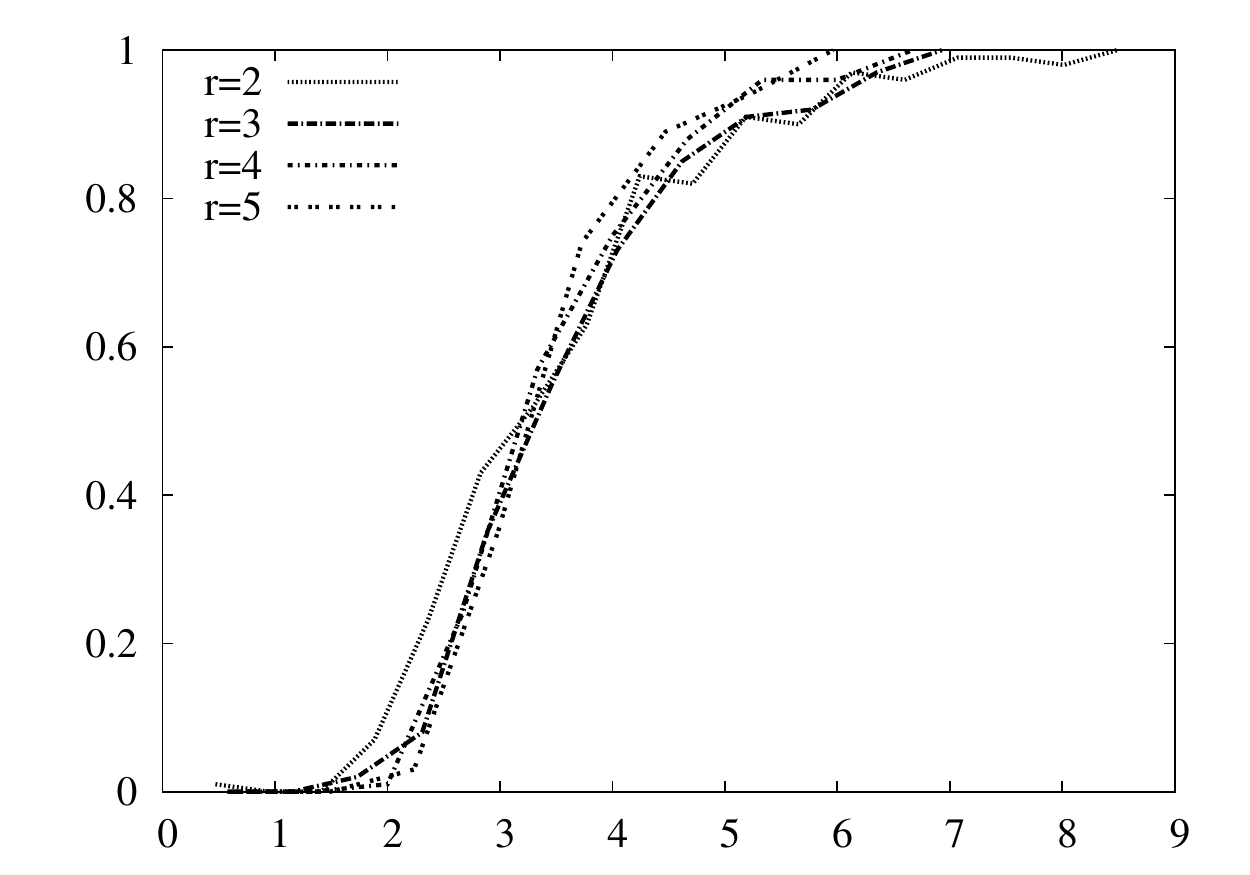}
	\put(-60,-5){$\alpha$}
	\includegraphics[width=.3\textwidth]{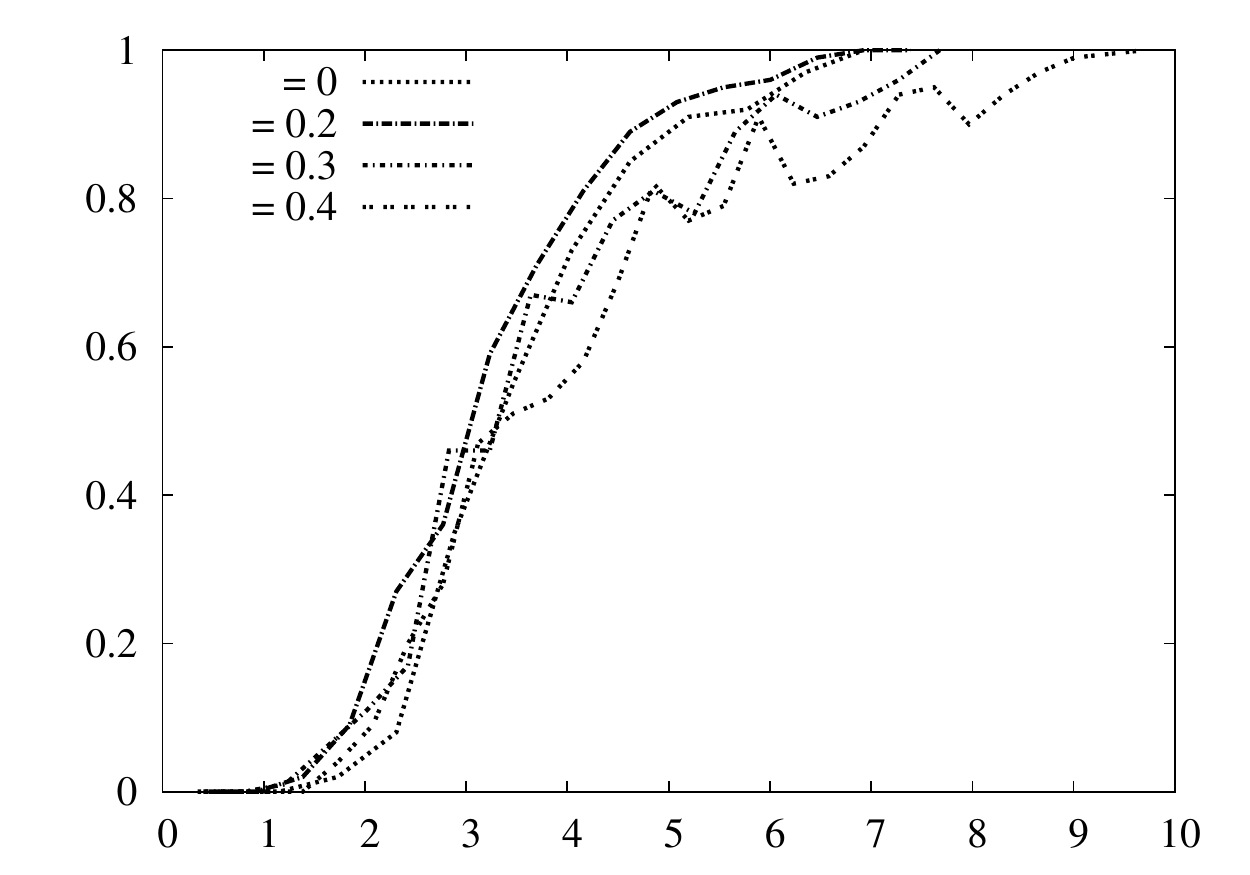}
	\put(-101,62){\tiny$\rho$}
	\put(-101,66){\tiny$\rho$}
	\put(-101,70){\tiny$\rho$}
	\put(-97,74){\tiny$\rho$}
	\put(-60,-5){$\alpha$}
\end{center}
\caption{Average recovery rate converges to a universal curve over $\alpha$ when $p=\alpha  r^{1/2}  \ln n/( (1-\rho)n^{3/2}) $, 
where $\rho=\max_{i\neq j\in[r]} \langle \u_i,\u_j\rangle$ and $r=O(\sqrt{n})$. }
\label{fig:threshold}
\end{figure}
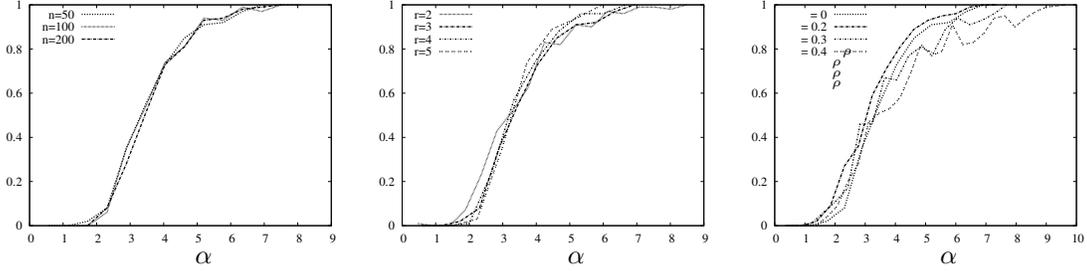

%% file: results.tex
\section{Analysis of the Alternating Minimization Algorithm}
\label{sec:results} 
In this section, we provide a proof of Theorem \ref{thm:main} and 
the proof sketches of the required main technical theorems. 
Formal proofs of the technical theorems and lemmas are provided in the appendix. 
There are two key components: $a$) the analysis of the initialization step (Section \ref{sec:resultinit}); and 
$b$) the convergence of alternating minimization given a sufficiently accurate initialization (Section \ref{sec:resultam}). 
We use these two analyses to prove Theorem \ref{thm:main} in Section \ref{sec:resultmain}.

\subsection{Initialization Analysis}
\label{sec:resultinit}
\input{result_init}

\subsection{Alternating Minimization Analysis}
\label{sec:resultam}
We now provide convergence analysis for the alternating minimization part of 
Algorithm~\ref{algo:rkk} to recover rank-$r$ tensor $T$. 
 Our analysis assumes that 
$\|\ui-\uoi\|_2\leq c\sigma_{\rm min}/r\sigma_{\rm max}, \forall i$ where $c$ is a small constant 
(dependent on $r$ and the condition number of $T$). 
The above mentioned assumption can be satisfied using 
our initialization analysis and by assuming $\Omega$ is large-enough. 

At a high-level, our analysis shows that each step of Algorithm~\ref{algo:rkk} ensures geometric decay of a distance function (specified below) which is ``similar'' to $\max_{j}\|\uj^t-\uoj\|_2$. 

Formally, let $T=\sum_{\ell=1}^r\sol\cdot \uol \otimes \uol \otimes \uol.$ WLOG, we can assume that that $\sigma^*_\ell\leq 1$. 
Also, let $[U,\ \Sigma]=\{(\ul, \sll), 1\leq \ell \leq r\},$ be the $t$-th step iterates of Algorithm~\ref{algo:rkk}.
We assume that $\uo_\ell, \forall \ell$ are $\mu$-incoherent and 
$\ul, \forall \ell$ are $2\mu$-incoherent. 
Define, $\sdl=\frac{|\sll-\sol|}{\sol}$, $\ul=\uol+\dul$, $(\sdl)^{t+1}=\frac{|\sll^{t+1}-\sol|}{\sol}$, and $\ul^{t+1}=\uol+\dul^{t+1}$. 
Now, define the following distance function: 
$$d_\infty([U, \Sigma], [U^*, \Sigma^*])\;\equiv\; \max_{\ell}\left(\|\dul\|_2+\sdl\right)\;.$$ 
The next theorem shows that this distance function decreases geometrically with number of iterations of Algorithm~\ref{algo:rkk}.
A proof of this theorem is provided in Appendix \ref{app:rkk}. 
\begin{theorem}
\label{thm:rk}
If $d_\infty([U, \Sigma], [U^*, \Sigma^*])\leq \frac{1}{1600 r} \frac{\sigma^*_{min}}{\sigma^*_{max}}$ and $\ui$ is $2\mu$-incoherent for all $1\leq i\leq r$, then there exists a positive constant $C$ such that 
for $p\geq \frac{C r^2 (\sigma^*_{\rm max})^2 \mu^3 \log^2 n}{(\sigma_{\rm min}^*)^2 n^{3/2}}$  
 we have w.p. $\geq 1-\frac{1}{n^{7}}$, 
$$d_\infty([U^{t+1}, \Sigma^{t+1}], [U^*, \Sigma^*])\leq \frac{1}{2}d_\infty([U, \Sigma], [U^*, \Sigma^*]), $$
where $[U^{t+1},\ \Sigma^{t+1}]=\{(\ul^{t+1}, \sll^{t+1}), 1\leq \ell\leq r\}$ are the $(t+1)$-th step iterates of Algorithm~\ref{algo:rkk}. Moreover, each $\u^{t+1}_\ell$ is $2\mu$-incoherent for all $\ell$. 
\end{theorem}
\begin{figure}[h]
 \begin{center}
  \includegraphics[width=.3\textwidth]{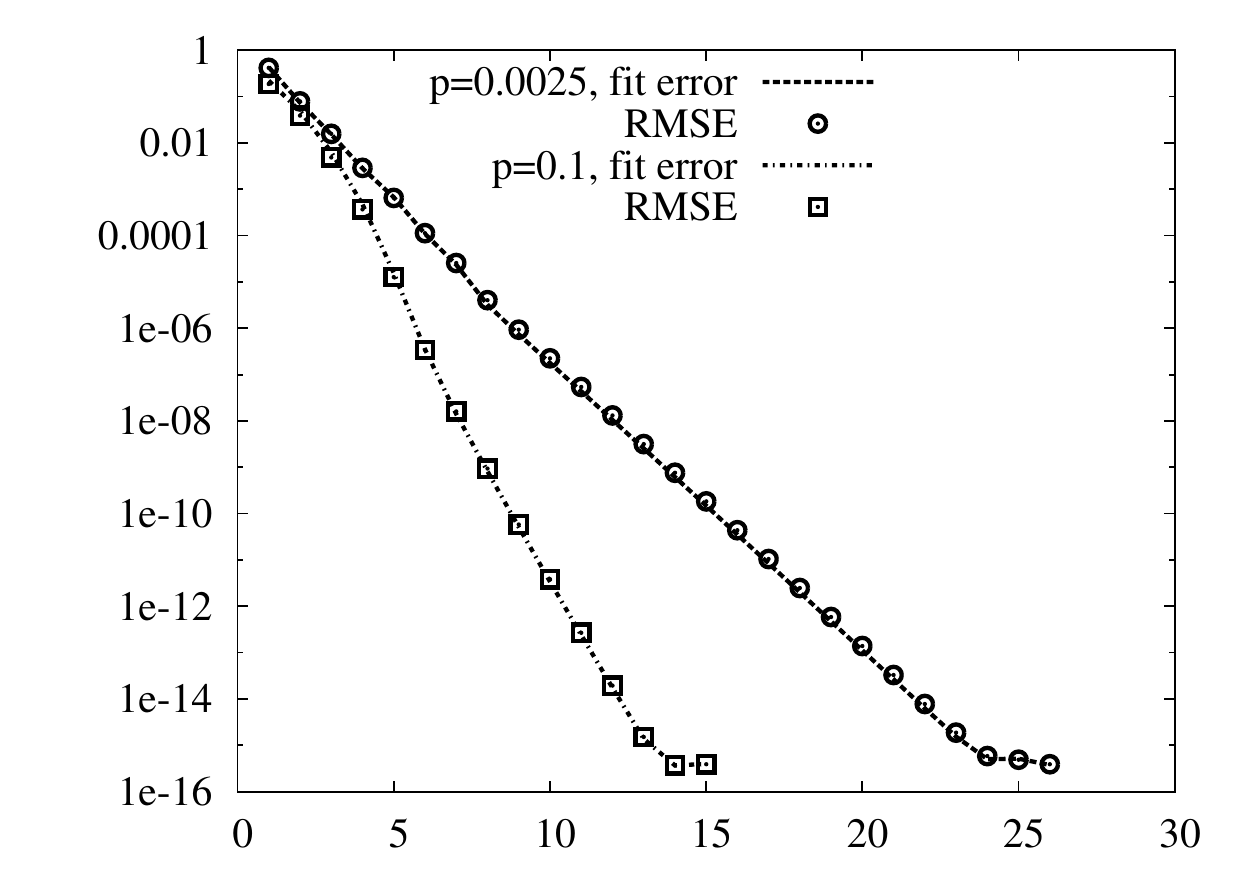}
  \put(-71,-6){iterations}
  \put(-132,50){error}
 \end{center}
 \caption{Algorithm \ref{algo:rkk} exhibits linear convergence until machine precision. 
 For the estimate $\hT_t$ at the $t$-th iterations, 
 the fit error $\|\cP_\Omega(T-\hT_t)\|_F/\|\cP_\Omega(T)\|_F$ 
 closely tracks the normalized root mean squared error $\|T-\hT_t\|_F/\|T\|_F$, 
 suggesting that it serves as a good stopping criterion.}
 \label{fig:iter}
\end{figure}
 Note that our number of samples depend on the number of iterations $\tau$. But due to linear convergence, our sample complexity increases only by a factor of $\log(1/\epsilon)$ where $\epsilon$ is the desired accuracy.

{\em Difference from Matrix AltMin}: Here, we would like to highlight differences between our analysis and analysis of the alternating minimization method for matrix completion (matrix AltMin) \cite{JNS13, hardt2013provable}. In the matrix case, the singular vectors $\uoi$'s need not be unique. Hence, the analysis is required to guarantee a decay in the  
 subspace distance $dist(U, U^*)$; typically, principal angle based subspace distance is used for analysis. 
In contrast, orthonormal $\uoi$'s uniquely define the tensor and hence one can obtain distance bounds $\|\ui-\uoi\|_2$  for each component $\ui$ individually. 

On the other other hand, 
an iteration of the matrix AltMin iterates over all the vectors $\ui, 1\leq i\leq r$, where $r$ is the rank of the current iterate and hence don't have to consider the error in estimation of the fixed components $U_{[r]\backslash q}=\{\ul,\ \forall\ \ell\neq q\}$, which is a challenge for the analysis of Algorithm~\ref{algo:rkk} and requires careful decomposition and bounds of the error terms. 



\subsection{Proof of Theorem \ref{thm:main}} 
\label{sec:resultmain} 
Let $T=\sum_{q=1}^r \sigma^*_q(u^*_q\otimes u^*_q\otimes u^*_q)$.
Denote the initial estimates $U^0=[u^0_1,\ldots,u^0_r]$ and $\sigma^0=[\sigma^0_1,\ldots,\sigma^0_r]$ 
to be the output of robust tensor power method at step 5 of Algorithm \ref{algo:rkk}.
With a choice of $p\geq C(\sigma_{\rm max}^*)^4 \mu^6 r^4 (\log n)^4/(\sigma_{\rm min}^*)^4n^{3/2}$ as per our assumption, 
Lemma~\ref{lem:tcinit} ensures that we have $\|u^0_q-u^*_q\| \leq \sigma^*_{\rm min}/(4800\,r\sigma_{\rm max})$ 
and $|\sigma_q^0-\sigma_q^*|\leq |\sigma_q^*|\sigma^*_{\rm min}/(4800\,r\sigma_{\rm max})$ with probability at least $1-n^{-5}$. 
This requires 
running robust tensor power method for $(r\log n)^c$ random initializations for some positive constant $c$, 
each requiring $O(|\Omega|)$ operations ignoring logarithmic factors. 

To ensure that we have sufficiently incoherent initial iterate, we perform thresholding proposed in \cite{JNS13}. 
In particular, we threshold all the elements of $\ui^0$ (obtained from RTPM method, see Step 3 of Algorithm~\ref{algo:rkk}) that are larger (in magnitude) than $\mu/\sqrt{n}$ to be $\frac{sign(\ul(i))\mu}{\sqrt{n}}$ and then re-normalize to obtain $\ui$. Using Lemma~\ref{lem:threshold}, this procedure ensures that the obtained initial estimate $\ui$ satisfies the two criteria that is required by Theorem~\ref{thm:rk}: a) $\|\u_i-\uoi\|_2\leq \frac{1}{1600 r}\cdot \frac{\sigma^*_{min}}{\sigma^*_{max}}$, and b) $\ui$ is $2\mu$-incoherent. 

With this initialization, Theorem~\ref{thm:rk} tells us that $O(\log_2(4r^{1/2}\|T\|_F/\varepsilon)$ 
iterations (each iteration requires $O(r|\Omega|)$ operations) 
is sufficient to achieve: 
\begin{eqnarray*}
	\| u_q-u^*_q\|_2 \leq  \frac{\varepsilon}{4r^{1/2}\|T\|_F} \;\text{ and }\;\;
	|\sigma_q - \sigma_q^*| \leq \frac{|\sigma_q^*|\varepsilon}{4r^{1/2}\|T\|_F}\;,
\end{eqnarray*}
for all $q\in[r]$ with probability at least $1-n^{-7}\log_2(4r^{1/2}\|T\|_F/\varepsilon)$. 
The desired bound follows from the next lemma with a choice of $\tepsilon={\varepsilon}/{4r^{1/2}\|T\|_F}$. 
For a proof we refer to Appendix \ref{app:tensordist}.
\begin{lemma}
	\label{lem:tensordist}
	For an orthogonal rank-$r$ tensor $T=\sum_{q=1}^r \sigma^*_q(u^*_q\otimes u^*_q \otimes u^*_q)$ 
	and any rank-$r$ tensor $\hT=\sum_{q=1}^r \sigma_q(u_q\otimes u_q \otimes u_q)$ satisfying 
	$\|u-u^*\|_2 \leq \tepsilon$ and $|\sigma-\sigma^*| \leq |\sigma^*|\tepsilon$ for all $q\in[r]$ 
and for all positive $\tepsilon>0$,
	we have $\|T-\hT\|_F \leq 4\,r^{1/2}\,\|T\|_F\,\tepsilon$.
\end{lemma}


%% file: result_init.tex
We first show that $(1/p)\cP_\Omega(T)$ is close to $T$ in spectral norm, 
and use it bound the error of robust power method applied directly to $\cP_\Omega(T)$. 
The normalization by $(1/p)$ compensates for the fact that many entries are missing. 
For a proof of this theorem, we refer to Appendix \ref{app:init}.
\begin{thm}[Initialization]
	\label{thm:tcinit}
	For $p=\alpha/n^{3/2}$ satisfying $\alpha \geq\log n$, 
	there exists a positive constant $C>0$ such that, with probability at least $1-n^{-5}$, 
	\begin{eqnarray}
		\frac{1}{\Tmax\,n^{3/2}\,p}\|\cP_{\Omega}(T) - p\,T \|_{2} &\leq& \frac{C\,(\log n)^2} {\sqrt{\alpha}} \;,
		\label{eq:tcinit}
	\end{eqnarray}
	where $\Tmax \equiv \max_{i,j,k} T_{ijk}$, and $\|T\|_2\equiv\max_{\|u\|=1}T[u,u,u]$ is the operator norm.
\end{thm}
Notice that $\Tmax$ is the maximum entry in the tensor $T$ and 
the factor $1/(\Tmax n^{3/2}p)$ corresponds to normalization with the worst case operator norm of $p\,T$, 
since $\|pT\|_2\leq \Tmax n^{3/2} p$ and the maximum is achieved by $T=\Tmax (\ones \otimes \ones \otimes \ones)$.  
The following theorem guarantees that $O(n^{3/2} (\log n)^2)$ samples are sufficient  to ensure that we get arbitrarily small error. 
A formal proof is provided in Appendix. 

Together with an analysis of  {\em robust tensor power method} \cite[Theorem 5.1]{AGHKT12},  
the next error bound follows from directly substituting \eqref{eq:tcinit} and using the fact that for incoherent tensors 
$\Tmax \leq \sigma_{\rm max} \mu(T)^{3} r/n^{3/2}$.
Notice that the estimates can be computed efficiently, requiring only $O(\log r + \log \log \alpha)$ iterations, each iteration requiring 
$O(\alpha n^{3/2} )$ operations. This is close to the time required to read the $|\Omega| \simeq \alpha n^{3/2}$ samples. One caveat is that 
we need to run robust power method $\poly(\r \log n)$ times, each with fresh random initializations. 

\begin{lemma} 
	\label{lem:tcinit}
	For a $\mu$-incoherent tensor with orthogonal decomposition 
$T=\sum_{\ell=1}^r \sigma^*_\ell (\u^*_\ell\otimes \u^*_\ell\otimes \u^*_\ell)\in\reals^{n\times n\times n}$, 
	there exists positive numerical constants $C, C'$ such that  
	when $\alpha \geq C (\sigma_{\rm max}/\sigma_{\rm min})^2 r^{5}\mu^6(\log n)^4$, running 
	$C'(\log r + \log \log \alpha)$ iterations of the robust tensor power method applied to $\cP_\Omega(T)$ achieves 
	\begin{eqnarray*}
		\| \u_\ell^* - {\u}_\ell^0\|_2 &\leq& C' \frac{\sigma^*_{\rm max}}{|\sigma^*_\ell|} \frac{\mu^3 r(\log n)^2 }{\sqrt{\alpha}} \;, \\ 
		\frac{| \sigma_\ell^* - \sigma_\ell |}{|\sigma^*_\ell|} &\leq& C' \frac{\sigma^*_{\rm max}}{|\sigma^*_\ell|} \frac{\mu^3 r(\log n)^2 }{\sqrt{\alpha}}  \;,
	\end{eqnarray*}
	for all $\ell\in[r]$ with probability at least $1-n^{-5}$, 
	where $\sigma^*_{\rm max} = \max_{\ell\in[\r]} |\sigma^*_\ell|$ and 
	$\sigma^*_{\rm min} = \min_{\ell\in[\r]} |\sigma^*_\ell|$. 
\end{lemma}




%% file: discussion.tex
\subsection{Fundamental limit and random hypergraphs} 
For matrices, it is known that exact matrix completion is impossible if the underlying graph is disconnected. 
A refinement of this analysis for Erd\"os-Ren\'yi graphs 
provides a lower bound on matrix completion: when sample size is less than $C\mu r n \log n$, 
no algorithm can recover the original matrix \cite{CT10}.

However, for tensor completion and random hyper graphs, such a simple connection does not exist. 
It is not known how the properties of the hyper graph is related to recovery. 
In this spirit, a rank-one third-order tensor completion has been studied in a specific context of 
{\em MAX-3LIN problems}. 
Consider a series of linear equations over 
$n$ binary variables $x=[x_1\,\ldots\, x_n]\in\{\pm1\}^n$.  
An instance of a 3LIN problem consists of a set of linear equations on GF(2), 
where each equation involve exactly three variables, e.g.  
\begin{eqnarray}
	x_1 \oplus x_2\oplus x_3  &= &+1\nonumber\\
	x_2 \oplus x_3\oplus x_4  &= &-1\nonumber\\
	x_3 \oplus x_4\oplus x_5  &= &+1\label{eq:ex1}
\end{eqnarray}
We use $-1$ to denote true (or $1$ in GF(2)) and $+1$ to denote false (or $0$ in GF(2)). 
Then the exclusive-or operation denoted by $\oplus$ is the integer multiplication. 
the MAX-3LIN problem is to find a solution $x$ that satisfies as many number of equations as possible. 
This is an NP-hard problem in general, and hence  random instances of the problem with a 
{\em planted solution} has been studied \cite{BO09}.
%
Algorithm \ref{algo:rkk} provides a provable guarantee for MAX-3LIN with random assignments. 
\begin{coro}
	For random MAX-3LIN problem with a planted solution, 
	under the hypotheses of Theorem \ref{thm:main},    
	Algorithm \ref{algo:rkk} finds the correct solution 
	with high probability.
\end{coro}
Notice that the tensor has incoherence one and rank one. 
This implies exact reconstruction for $P\geq C (\log n )^4 / n^{3/2}$. 
This significantly improves over 
a message-passing approach to MAX-3LIN in \cite{BO09}, 
which is guaranteed to find the planted solution for 
$p\geq C(\log \log n)^2/(n\log n)$.
 It was suggested that a new notion of connectivity called {\em propagation connectivity} 
 is a sufficient condition for the solution of random MAX-3LIN problem with a planted solution to be unique 
 \cite[Proposition 2]{BO09}. 
 Precisely, it is claimed that if the hypergraph corresponding to an instance of MAX-3LIN is propagation connected, then 
 the optimal solution for MAX-3LIN is unique and there is an efficient algorithm that finds it. 
 However,  the example in \ref{eq:ex1} is propagation connected but there is no unique solution: 
 $[1,1,1,-1,-1],[1,-1,-1,1,-1],[-1,1,-1,1,-1],[-1,-1,1,1,1]$ all satisfy the equations and corresponding tensor $T=x\otimes x\otimes x$ is not uniquely defined either. 
 This proves that propagation connectivity is not a sufficient condition for uniqueness of the MAX-3LIN solution.

\subsection{Open Problems and Future Directions} 
\label{sec:discussion} 

{\em Tensor completion for non-orthogonal decomposition.} 
Numerical simulations suggests that non-orthogonal CP models can be recovered exactly 
(without the usual whitening step). It would be interesting to analyze our algorithm under non-orthogonal CP model. However, we would like to point here that even with fully observed tensor, exact factorization is known only for orthonormal tensors. Now, given that our method guarantees not only completion but also factorization of the tensor (which is essential for large scale applications), it is natural that our method would require a similar condition. 


{\em Optimal dependence on $r$.} 
The numerical results suggests threshold sample size scaling as $\sqrt{r}$. 
This is surprising since the degrees of freedom in describing CP model scales as linearly in $r$. 
This implies that the $\sqrt{r}$ scaling can only hold for $r=O(\sqrt{n})$. 
In comparison, for matrix completion we know that the threshold scales as $r$. 
It would be important to understand why this change in dependence in $r$ happens for higher order tensors, 
and identify how it depends on $k$ for $k$-th order tensor completion.

%% file: appendix.tex
\appendix
\input{app_init}
\input{app_altmin}

%% file: app_init.tex
\section{Proof of Theorem \ref{thm:tcinit} for Initialization Analysis}
\label{app:init}

We prove the following bound on the spectrum of random tensors:
\begin{eqnarray*} 
	 \max_{x,y,z, \|x\|=\|y\|=\|z\|=1} \; \big( \cPO(T) - p\, T\big)[x,y,z] &\leq& C\, \Tmax\, (\log n)^2\, \sqrt{(n_1 \,n_2 \,n_3)^{1/2}\, p}\;.
 \end{eqnarray*}
 Here we prove the theorem for general case where $T$ 
 is not symmetric and might even have different dimensions $n_1$, $n_2$ and  $n_3$. 
 Inspired by \cite{FKS89,KMO10}, our strategy is as follows: 
 \begin{itemize}
 	\item[(1)] Reduce to $x$,$y$, and $z$ which belongs to {\em discretized sets} $\tS_{n_1}$, $\tS_{n_2}$, and $\tS_{n_3}$;
	\item[(2)] Bound the contribution of {\em light triples} using concentration of measure;
	\item[(3)] Bound the contribution of {\em heavy triples} using the discrepancy property of a random tripartite hypergraph. 
\end{itemize}

Define a discretization of an $n$-dimensional ball as
\begin{eqnarray*}
	\tS_n &\equiv& \Big\{ x\in \Big\{\frac{\Delta}{\sqrt{n}}\int \Big\}^n \;:\; \|x\|\leq 1\Big\} \;,
\end{eqnarray*}
such that $\tS_n \subseteq S_n \equiv \{x\in\reals^n\,:\, \|x\|\leq 1\}$. Later we will set $\Delta$ to be a small enough constant. 
\begin{lemma}[Remark 4.1 in \cite{KMO10}]
	For any tensor $A \in\reals^{n_1\times n_2\times n_3}$, 
	\begin{eqnarray*}
		\max_{x\in S_{n_1}, y\in S_{n_2}, x\in S_{n_3}} A[x,y,z] &\leq& \max_{x\in \tS_{n_1}, y\in \tS_{n_2}, x\in \tS_{n_3}} \frac{1}{(1-\Delta)^3}A[x,y,z]
	\end{eqnarray*}
	\label{lem:disc}
\end{lemma}
It is therefore enough show that the bound holds for discretized vectors all discretized vectors $x$, $y$, and $z$.  
One caveat is that such a probabilistic bound must hold with probability sufficiently close to one such that 
we can apply the union bound 
over all discretized choices of $x$, $y$, and $z$. 
The following lemma bounds the number of such choices. 
\begin{lemma} [\cite{KMO10}] The size of the discretized set is bounded by $|\tS_n| \leq \big( \Delta / 10 \big)^n$. 
	\label{lem:size}
\end{lemma}
A naive approach to upper bound $(\cPO(T)-p\,T)[x,y,z]$ would be to consider it as a random variable and apply concentration inequalities directly. 
However, this naive approach fails since $x$, $y$ and $z$ can contain entries that are much larger than their typical value of $O(1/\sqrt{n})$. 
We thus separate the analysis into two contributions, and apply concentration inequalities to bound the contribution of  the {\em light triples} and 
use graph topology of the random sampling to bound the contribution of the {\em heavy triples}. 
Define the light triples as 
\begin{eqnarray}
	\cL &\equiv& \left\{ \;(i,j,k) \,:\, \big| T_{ijk} x_i y_j z_k \big| \,\leq\, \Tmax \sqrt{\frac{\eps}{n_1 \,n_2\, n_3 }} \; \right\} \;. 
	\label{eq:deflight}
\end{eqnarray}
Heavy triples are defined as its complement $\bcL = \{[n_1]\times[n_2]\times[n_3]\} \setminus \cL$. 
Later we will set the appropriate value for $\eps=\Theta(p \sqrt{n_1n_2n_3})$. 
We can then write each contributions separately as 
\begin{eqnarray}
	\Big| \big(\cPO(T)-p\,T \big)[x,y,z] \Big| &\leq& \Big| \sum_{(i,j,k)\in\cL} \big( \cPO(T)_{ijk} x_iy_jz_k\big) - p\,T[x,y,z] \Big| + \Big| \sum_{(i,j,k)\in\bcL} \cPO(T)_{ijk} x_iy_jz_k \Big|\;. \label{eq:tcinit_separate}
\end{eqnarray}
We will prove that both contributions are upper bounded by $C \Tmax (\log n)^2\sqrt{(n_1n_2n_3)^{1/2} p} $ with some positive constant $C$ 
for all $x\in \tS_{n_1}$, $y\in\tS_{n_2}$, and $z\in\tS_{n_3}$. 
The bound on the light triples follows from Chernoff's concentration inequalities. 
The bound on the heavy triples follows from the discrepancy property of random hyper graphs,  
which implies that there cannot be too many triples with large contributions. 
Theorem \ref{thm:tcinit} then follows from Lemma \ref{lem:disc} with an appropriate choice of $\Delta = \Theta(1)$. 

\subsection{Bounding the contribution of light triples}
\label{sec:light}


Let $Z\equiv  \sum_{(i,j,k)\in\cL} \big( \cPO(T)_{ijk} x_iy_jz_k\big) - p\,T[x,y,z]$ for some $x\in S_{n_1}$, $y\in S_{n_2}$, and 
$z\in S_{n_3}$. 
We claim that 
\begin{eqnarray}
	\prob\Big( Z>  \frac{p\Tmax\sqrt{n_1n_2n_3}}{\sqrt{\eps}}  + t\,(n_1+n_2+n_3)\, \frac{2\Tmax \sqrt{\eps}}{\sqrt{n_1n_2n_3}} \Big) &\leq& \exp\big\{ - t(n_2+n_2+n_3)\big\}\;. 
	\label{eq:lightconcentration}
\end{eqnarray}
We first show that the mean of $Z$ is bounded as 
\begin{eqnarray}
	\big| \E[Z] \big| & \leq & 2\,p\,\Tmax\, \sqrt{\frac{n_1n_2n_3}{\eps}}  \;.\label{eq:lightmean}
\end{eqnarray}
The mean can be written as 
$\E[Z] = p \sum_{\cL} T_{ijk} x_i y_j z_k - p\sum_{[n_1]\times[n_2]\times[n_3]} T_{ijk} x_i y_j z_k = p\sum_{\bcL} T_{ijk} x_i y_j z_k $. 
Using the fact that for heavy triples $|T_{ink}x_iy_jz_k| \geq \Tmax\sqrt{\eps/(n_1n_2n_3)}$, the expected contribution is then bounded by 
\begin{eqnarray*}
	\Big| \sum_{(i,j,k)\in \bcL} T_{ijk} x_i y_j z_k \Big| &\leq& \sum_{(i,j,k)\in \bcL} \frac{T_{ijk}^2 x_i^2 y_j^2 z_k^2}{|T_{ijk}x_i y_j z_k|}\\
		&\leq& \frac{\sqrt{n_1n_2n_3}}{\Tmax\sqrt{\eps}}\sum_{(i,j,k)\in \bcL} T_{ijk}^2 x_i^2 y_j^2 z_k^2 \\ 
		&\leq& \frac{\Tmax\sqrt{n_1n_2n_3}}{\sqrt{\eps}} \;.
\end{eqnarray*}
We next show concentration of $Z$ around item mean. 
Let $\lambda = \sqrt{n_1n_2n_3} / (2\Tmax \sqrt{\eps})$ such that 
$|\lambda T_{ijk}x_iy_jz_k| \leq 1/2$ for all $(i,j,k)\in \cL$. 
Then, $e^{\lambda T_{ijk}x_iy_jz_k}-1 \leq \lambda T_{ijk}x_iy_jz_k + 2\lambda (T_{ijk}x_iy_jz_k)^2 $. 
\begin{eqnarray*}
	\E[e^{\lambda Z}] &=& \exp \{ - \lambda \,p\, T[x,y,z]\}\,\prod_{(i,j,k)\in\cL} \big( 1-p+p\,e^{\lambda T_{ijk}x_iy_jz_k } \big) \\ 
		&\leq& \exp \{ - \lambda \,p\, T[x,y,z]\}\,\prod_{(i,j,k)\in\cL}\big( 1+p\,(\lambda  T_{ijk}x_iy_jz_k  + 2\lambda^2( T_{ijk}x_iy_jz_k )^2  ) \big)\\
		&\leq& \exp \big\{ p\sum_{\cL} \lambda  T_{ijk}x_iy_jz_k - \lambda \,p\, T[x,y,z] + p\sum_\cL  2\lambda^2( T_{ijk}x_iy_jz_k )^2\big\} \\
		&\leq& \exp\Big\{\lambda\E[Z] +\frac{\,p\,n_1n_2n_3}{2 \eps}\Big\} \;. 
\end{eqnarray*}
Applying Chernoff bound $\prob(Z-\E[Z]>t)\leq \E[e^{\lambda Z}]e^{-\lambda\E[Z]-\lambda t}$, 
this proves \eqref{eq:lightconcentration}. Note that the deviation of $-Z$ can be bounded similarly. 
We can now finish the proof of upper bound on the contribution of light triples by taking the union 
bound over all discretized vectors inside the ball. Setting $t=2\log(20/\Delta)$ in \eqref{eq:lightconcentration}, we get   
\begin{align*}
	&\prob\Big( \max_{x\in \tS_{n_1}, y\in\tS_{n_2},z\in\tS_{n_3}} \sum_{(i,j,k)\in\cL} T_{ijk}x_iy_jz_k - p\,T[x,y,z] \geq 
	\frac{p\Tmax\sqrt{n_1n_2n_3}}{\sqrt{\eps}}  + 2\log(20/\Delta)\,(n_1+n_2+n_3)\, \frac{2\Tmax \sqrt{\eps}}{\sqrt{n_1n_2n_3}} 
	\Big) \\
	&\leq 
	2\,e^{(n_1+n_2+n_3)\log(20/\Delta)}\, e^{-2\log(20/\Delta)(n_1+n_2+n_3)}\\
	&\leq 2e^{-(n_1+n_2+n_3)\log(20/\Delta)}\;.
\end{align*}
Since $p=\eps/\sqrt{n_1n_2n_3}$, this proves that the contribution of light triples is bounded by $C\, \Tmax \sqrt{\eps} $
with high probability.

Note that for the range of $p=\eps/n^2$, the contribution of light couples is bounded by $C\,\Tmax\sqrt{\eps/n}$. 
However, even in this regime of $p$, the contribution of heavy triples is still $\Omega(1)$, which dominates the light triples by 
a factor of $\sqrt{n}$. This is the reason for the choice of $p=\Theta({\eps/n^{1.5}})$.

\subsection{Bounding the contribution of heavy triples}
\label{sec:heavy}
The contribution of heavy triples is bounded by 
\begin{eqnarray*}
	\Big| \sum_{(i,j,k)\in \bcL} T_{ijk}x_iy_jz_k \Big| &\leq& \Tmax \sum_{(i,j,k)\in\bcL} |x_iy_jz_k| \;.
\end{eqnarray*}
In the following, we will show that the right-hand side of the above inequality is upper bounded by 
\begin{eqnarray*}
	 \sum_{(i,j,k)\in\bcL} |x_iy_jz_k| &\leq& C \sqrt{\eps} (\log n)^2\;,
\end{eqnarray*}
for some positive numerical constant $C>0$ with probability larger than $1-n^{-5}$.

We consider a hypergraph $G=([n_1]\times[n_2]\times[n_3],E)$ with undirected hyper edges, 
where each edge connects three nodes, each one from each set $[n_1]$, $[n_2]$, and $[n_3]$. 
Given a sampling of entries in a tensor, 
we let the edges in $G$ denote the positions of the entries that is sampled. 
The proof is a generalization of similar proof for matrices in \cite{FKS89,FeO05,KMO10} and is based on 
two properties of the hypergraph $G$.
Define the degree of a node as the number of edges connected to that particular node such that 
$\deg_1(i) \equiv |\{(i,j,k)\in E \} |$, and similarly define $\deg_2(j)$ and $\deg_3(k)$. 
Define the degree of two nodes as the number of edges connected to both of the nodes such that 
$\deg_{12}(i,j) \equiv |\{(i,j,k)\in E \} |$, and similarly define $\deg_{13}(i,k)$ and $\deg_{23}(j,k)$. 

\begin{itemize}
	\item [1.] {\em Bounded degree property.} A hyper graph $G$ satisfies the bounded degree property if 
		the degree are upper bounded as follows: 
		\begin{eqnarray}
			\deg_1(i) &\leq& \xi_0\,p\,n_2n_3 \hspace{0.5cm}\text{ for all }i\in[n_1]\;,\nonumber\\
			\deg_2(j) &\leq& \xi_0\,p\,n_1n_3 \hspace{0.5cm} \text{ for all }j\in[n_2]\;, \nonumber\\
			\deg_3(k) &\leq& \xi_0\,p\,n_1n_2 \hspace{0.5cm} \text{ for all }k\in[n_3]\;,\nonumber\\
			\deg_{12}(i,j) &\leq& \xi_0\,(p\,n_3+\log n_3) \hspace{0.5cm} \text{ for all }i\in[n_1],j\in[n_2]\;,\nonumber\\
			\deg_{13}(i,k) &\leq& \xi_0\,(p\,n_2+\log n_2) \hspace{0.5cm} \text{ for all }i\in[n_1],k\in[n_3]\;, \nonumber\\
			\deg_{23}(j,k) &\leq& \xi_0\,(p\,n_1+\log n_1) \hspace{0.5cm} \text{ for all }j\in[n_2],k\in[n_3]\;,
			\label{eq:boundeddegree}
		\end{eqnarray}
		for some positive numerical constant $\xi_0>0$ (independent of $n_1,n_2,n_3$ and $p$) 
		where $p=|E|/(n_1n_2n_3)$.
	\item[2.] {\em Discrepancy property.} A hyper graph $G$ satisfies the discrepancy property if for any subset of nodes $A_1\in[n_1]$, 
			$A_2\in[n_2]$, and $A_3\in[n_3]$, at least one of the following is true: 
		\begin{eqnarray}
		e(A_1,A_2,A_3) &\leq& \xi_1 \, \be(A_1,A_2,A_3) \;, \label{eq:discrepancy1}\\ 
		e(A_1,A_2,A_3) \ln\Big(\frac{e(A_1,A_2,A_3)}{\be(A_1,A_2,A_3)}\Big) &\leq& \xi_2 \, \max\Big\{ |A_1|\ln\Big(\frac{e\,n_1}{|A_1|}\Big)\;,\; |A_2|\ln\Big(\frac{e\,n_2}{|A_2|}\Big)\;,\; |A_3|\ln\Big(\frac{e\,n_3}{|A_3|}\Big)\Big\}\, \;, \label{eq:discrepancy2}
		\end{eqnarray}
		for some positive numerical constants $\xi_1,\xi_2>0$ (independent of $n_1,n_2,n_3$ and $p$). 
		Here, $e(A_1,A_2,A_3)$ denotes the number of edges between the three subsets $A_1,A_2$ and $A_3$, 
		and $\be(A_1,A_2,A_3) \equiv p\,|A_1|\,|A_2|\,|A_3|$ denotes the average number of edges between the three subsets. 
\end{itemize}

We first prove that if the sampling pattern is defined by a graph $G$ which satisfies both the bounded degree and discrepancy properties, 
then the contribution of heavy triples is $O(\sqrt{\eps})$. Notice that this is a deterministic statement, that holds for all graphs with the above properties. We then finish the proof by showing that the random sampling satisfies both the bounded degree and discrepancy properties 
with probability at least $1-n^{-5}$. 

We partition the indices according to the value of corresponding vectors: 
\begin{eqnarray*}
	A_1^{(u)} &\equiv& \Big\{ i\in[n_1] \,:\, \frac{\Delta}{\sqrt{n_1}}2^{u-1} \leq|x_i| < \frac{\Delta}{\sqrt{n_1}}2^u \Big\} \\
	A_2^{(v)} &\equiv& \Big\{ j\in[n_2] \,:\, \frac{\Delta}{\sqrt{n_2}}2^{v-1} \leq|y_j| < \frac{\Delta}{\sqrt{n_2}}2^v \Big\} \\
	A_3^{(w)} &\equiv& \Big\{ k\in[n_3] \,:\, \frac{\Delta}{\sqrt{n_3}}2^{w-1} \leq|z_k| < \frac{\Delta}{\sqrt{n_3}}2^w \Big\} \;,
\end{eqnarray*}
for $u\in \{1,\ldots,\lceil \log_2(\sqrt{n_1}/\Delta)\rceil +1 \}$, $v\in\{1,\ldots,\lceil \log_2(\sqrt{n_2}/\Delta)\rceil +1  \}$, and $w\in\{1,\ldots,\lceil \log_2(\sqrt{n_3}/\Delta)\rceil +1 \}$. 
We denote the size of each set by  $a^{(u)}_i \equiv |A^{(u)}_i|$.  
We use $e_{uvw}$ to denote the number of edges between three subsets $A_1^{(u)}$, $A_2^{(v)}$, and $A_3^{(w)}$, 
and we use $\be_{uvw} \equiv p\,a_1^{(u)} a_2^{(v)} a_3^{(w)}  $ to denote the average number of edges. 
Notice that the above definition of $A_1^{(u)}$'s cover all non-zero values of the entries of $x$, since, 
with discretization, the smallest possible positive value is $\Delta/\sqrt{n_1}$. The same applies to the entries of $y$ and $z$. 

\begin{eqnarray*}
	\sum_{(i,j,k) \in \bcL}  \big| x_iy_jz_k \big| &\leq& \sum_{(i,j,k): |x_iy_jz_k|>\sqrt{\eps/(n_1n_2n_3)}}  \big| x_iy_jz_k \big| \\
		&\leq& \sum_{(u,v,w):2^{u+v+w}>8\sqrt{\eps}/\Delta^3 } \underbrace{e_{uvw} \frac{\Delta 2^u}{\sqrt{n_1}} \frac{\Delta 2^v}{\sqrt{n_2}} \frac{\Delta 2^w}{\sqrt{n_3}}}_{\sigma_{uvw}}\;.
\end{eqnarray*}
Note that since $\sum_u a_1^{(u)} 2^{2(u-1)}\Delta^2/n_1 \leq \|x\|^2 \leq 1$, we get that 
\begin{eqnarray}
	a_1^{(u)} &\leq& (n_1/\Delta^2)2^{-2(u-1) } \;, \nonumber \\
	a_2^{(v)} &\leq& (n_2/\Delta^2)2^{-2(v-1) } \;, \nonumber \\
	a_3^{(w)} &\leq& (n_3/\Delta^2)2^{-2(w-1) } \;. \label{eq:sizebound}
\end{eqnarray}
The contributions from various combinations of $(u,v,w)$ utilize various subsets of our assumptions. 
We prove that in each case the contribution is $O(\sqrt{\eps}(\log n)^2)$ as follows. 
\begin{itemize}
	\item[] {\bf Case1.} For $(u,v,w)$ satisfying the first discrepancy property \eqref{eq:discrepancy1} : $e_{uvw}\leq \xi_1 \be_{uvw}$. 
	
	In this case, using \eqref{eq:sizebound} and the fact that $p=\eps/\sqrt{n_1n_2n_3}$, 
	\begin{eqnarray*}
		\sum \sigma_{uvw} &\leq& \xi_1 p a_1^{(u)}a_2^{(v)}a_3^{(w)}\frac{\Delta^32^{u+v+w}}{ \sqrt{n_1n_2n_3}}\\	
			&\leq& \frac{64\xi_1 p\sqrt{n_1n_2n_3}} {\Delta^3 2^{u+v+w}}\\
			&\leq& 16\xi_1 \sqrt{\eps} (\log n)^2  \;,
	\end{eqnarray*}
	where $n\equiv \max \{n_1,n_2,n_3\} $ and in the last inequality 
	we used the fact that we are summing over heavy triples satisfying $\Delta^32^{u+v+w}>8\sqrt{\eps}$, 
	and $\sum_{(u,v,w):2^{u+v+w}\leq8\sqrt{\eps}/\Delta^3} 2^{-(u+v+w)} \leq 2\log_2(\sqrt{n_1}/\Delta)\, \log_2(\sqrt{n_2}/\Delta)\,\Delta^3/(8\sqrt{\eps})$. 
	\item[] {\bf Case2.} For $(u,v,w)$ satisfying the second discrepancy property in \eqref{eq:discrepancy2}.
	\begin{itemize}
		\item[] {\bf Case 2-1.} For $(u,v,w)$ satisfying $\ln(e_{uvw}/\be_{uvw}) \leq (1/2) \ln (en_3/a_3^{(w)})= (1/4)( \ln(en_3/(a_3^{(w)}2^{2w})) + \ln(2^{2w}) )$. 
		\begin{itemize} 
			\item[] {\bf Case 2-1-1.} When $\ln(2^{2w}) \leq  \ln (en_3/(a_3^{(w)}2^{2w})) $, we have 
			$\ln(e_{uvw}/\be_{uvw}) \leq \ln(en_2/(a_3^{(w)}2^{2w}))$, which gives 
			\begin{eqnarray*}
				e_{uvw} &\leq& e n_3 \be_{uvw} /(\a_3^{(w)}2^{2w}) \\
					&\leq& e\,a_1^{(u)}a_2^{(v)} n_3 p 2^{-2w}\\
					&\leq& \frac{16 e}{\Delta^4}\,n_1n_2n_3 p 2^{-2(u+v+w)} \;.
			\end{eqnarray*}
			It follows that $\sum\sigma_{uvw} \leq (16/\Delta)p\sqrt{n_1n_2n_3} 2^{-u-v-w} \leq 2 \Delta^2 \sqrt{\eps} (\log n)^2$ using the fact that we are summing over heavy triples. 
			
			\item[] {\bf Case 2-1-2.} When $\ln(2^{2w}) > \ln (en_3/(a_3^{(w)}2^{2w}))$, we have 
				$\ln(e_{uvw}/\be_{uvw}) \leq \ln(2^{w})$. 
			\begin{itemize}
				\item[] {\bf Case 2-1-2-1.} For $\sqrt{\eps}e_{uvw} > 2^{u+v+w} \be_{uvw}$, it follows that 
				$2^{u+v}\leq\sqrt{\eps}$. Since we are in the case where the first discrepancy does not hold, i.e. $e_{uvw}>\xi_1\be_{uvw}$, and the the second discrepancy property holds, we have 
				$e_{uvw}\leq e_{uvw}\ln(e_{uvw}/\be_{uvw}) \leq \xi_2 a_3^{(w)}\ln(en_3/a_3^{(w)}) \leq
					2 \xi_2 a_3^{(w)} \ln(2^{2w})$
				Then, 
				\begin{eqnarray*}
					\sum\sigma_{uvw} &\leq& \sum 2\xi_2a_3^{(w)}\ln(2^{2w}) \frac{\Delta^32^{u+v+w}}{\sqrt{n_1n_2n_3}}\\ 
					&\leq& \sum \frac{8\xi_2\Delta \sqrt{n_3} 2^{u+v} }{\sqrt{n_1n_2}}\frac{\ln(2^{w})}{2^w}\\
					&\leq& 8\Delta\xi_2 \sqrt{\frac{n_3}{n_1n_2}} \sqrt{\eps} \log_2(\sqrt{n_1}/\Delta) \log_2(\sqrt{n_3}/\Delta)\;,
				\end{eqnarray*}
				which is $O\big(\sqrt{\eps} (\log n)^2\sqrt{(n_3/(n_1n_2))}\,\big)$
				\item[] {\bf Case 2-1-2-2.} For $\sqrt{\eps}e_{uvw} \leq 2^{u+v+w} \be_{uvw}$, 
				\begin{eqnarray*}
					\sum\sigma_{uvw} &\leq& \sum \, \frac{2^{u+v+w} p\,a_1^{(u)}a_2^{(v)}a_3^{(w)}}{\sqrt{\eps}} \frac{\Delta^3 2^{u+v+w}}{\sqrt{n_1n_2n_3}} \\
					&\leq& \frac{\sqrt{\eps}}{\Delta^3} \sum\,  \frac{a_1^{(u)} a_2^{(v)} a_3^{(w)}\Delta^3 2^{2(u+v+w)}}{n_1n_2n_3}\\
					&\leq& \frac{\sqrt{\eps}}{\Delta^3} \|x\|^2\, \|y\|^2\, \|z\|^2\;,
				\end{eqnarray*}
				which is $O(\sqrt{\eps})$. 
			\end{itemize}
		\end{itemize}
		
		\item[] {\bf Case 2-2.} For $(u,v,w)$ satisfying $\ln(e_{uvw}/\be_{uvw}) >  (1/2) \ln (en_3/a_3^{(w)})$. 
		\begin{itemize}
			\item[] {\bf Case 2-2-1.} For $2^{u+v} \leq \sqrt{n_1n_2\eps/n_3} 2^w$, we know from the condition $\ln(e_{uvw}/\be_{uvw}) >  (1/2) \ln (en_3/a_3^{(w)})$, that $e_{uvw} \leq 2 \xi_2 a_3^{(w)}$. Then, 
			\begin{eqnarray*}
				\sum\sigma_{uvw} &\leq& \sum\,2\xi_2a_3^{(w)} \frac{2^{u+v+w}\Delta^3}{\sqrt{n_1n_2n_3}} \\
				&\leq& \sum \, 8\xi_2\Delta 2^{u+v-w} \sqrt{\frac{n_3}{n_1n_2}} \\
				&\leq& 8\xi_2\Delta \sqrt{\eps} \log_2(\sqrt{n_1}/\Delta)\log_2(\sqrt{n_2}/\Delta)\;,
			\end{eqnarray*}
			which is $O(\sqrt{\eps}(\log n)^2)$.
			\item[] {\bf Case 2-2-2.} For $2^{u+v} > \sqrt{n_1n_2\eps/n_3} 2^w$ 
			\begin{itemize}
				\item[] {\bf Case 2-2-2-1.} For $(u,v,w)$ satisfying bounded degree property with $\deg_{12}(i,j) \leq \xi_0 p n_3$, we have $e_{uvw}\leq a_1a_2\xi_0pn_3$. Then, 
				\begin{eqnarray*}
					\sum\sigma_{uvw} &\leq& \sum\, \frac{16\xi_0\eps}{\Delta}2^{w-u-v}\\
						&\leq& \sqrt{\frac{n_3}{n_1n_2}} \frac{16\xi_0 \sqrt{\eps}}{\Delta}\log_2(\sqrt{n_2}/\Delta)\log_2(\sqrt{n_2}/\Delta)  \;,
				\end{eqnarray*}
				which is $O(\sqrt{\eps}\sqrt{n_3/(n_1n_2)} (\log_2 n)^2)$. 
				\item[] {\bf Case 2-2-2-2.}	For $(u,v,w)$ satisfying bounded degree property with $\deg_{12}(i,j) \leq \xi_0 \log n_3$, we have $e_{uvw} \leq a_1a_2\xi_0 \log n$.
				\begin{eqnarray*}
					\sum\sigma_{uvw} &\leq& \sum \frac{16\xi_0\log n_3 2^{w-u-v}}{\Delta}\sqrt{\frac{n_1 n_2}{n_3}} \\
						&\leq& \frac{16\xi_0\log n_3 }{\Delta\sqrt{\eps}}\log_2(\sqrt{n_1}/\Delta)\log_2(\sqrt{n_2}/\Delta)\;,
				\end{eqnarray*}
				which is $O( (1/\sqrt{\eps}) (\log n)^3)$.
			\end{itemize}
		\end{itemize}
	\end{itemize}
\end{itemize}
For $\eps\geq \log n $, this proves that the contribution of the heavy triples is $O(\sqrt{\eps}(\log n)^2 )$.

We are left to prove that the bounded degree and the bounded discrepancy properties hold for 
a random tripartite hypergraph $G=(V_1\cup V_2 \cup V_3,E)$ where each edge is selected with probability $p$.
Precisely, let $n=\max\{|V_1|,|V_2|,|V_3|\}$, then the following lemma provides a bound on the degree and discrepancy, with high probability.
\begin{lemma}
	\label{lem:discrepancy}
	For any $\delta\in[0,1/e]$ and $p\geq (1/n^2)\log n$,  there exists numerical constants $C,C'>0$ such that 
	a random tripartite hyper graph satisfies 
	the bounded degree property: for all $i\in V_1$, $j\in V_2$, and $k\in V_3$, 
	\begin{eqnarray*}
		\deg_1(i) &\leq& 2pn_2n_3 +\frac83 \log \frac{3n_1}{\delta} \\
		\deg_2(j) &\leq& 2pn_1n_3 +\frac83 \log \frac{3n_2}{\delta} \\
		\deg_3(k) &\leq& 2pn_1n_2 +\frac83 \log \frac{3n_3}{\delta} \\
		\deg_{12}(i,j) &\leq& 2pn_3 +\frac83 \log \frac{3n_1n_2}{\delta} \\
		\deg_{13}(i,k) &\leq& 2pn_2 +\frac83 \log \frac{3n_1n_3}{\delta} \\
		\deg_{23}(j,k) &\leq& 2pn_1 +\frac83 \log \frac{3n_2n_3}{\delta} \\
	\end{eqnarray*}
	and the bounded discrepancy property: for all subsets $A_1\subseteq V_1$, $A_2\subseteq V_2$, and 
	$A_3 \subseteq V_3$, at least one of the following is true. 
	\begin{eqnarray*}
		e(A_1,A_2,A_3) &\leq& C\, \alpha^2 \,\be(A_1,A_2,A_3)\, \Big(1 + \frac{\ln(1/\delta) }{pn^2} \Big) \;,\text{ or } \\
		e(A_1,A_2,A_3) \ln\Big(\frac{e(A_1,A_2,A_3)}{\be(A_1,A_2,A_3)}\Big) &\leq& C' \Big( \ln\frac{\alpha}{\delta} + \max\Big\{|A_1|\ln\frac{e\,n_1}{|A_1|}\;,\; |A_2|\ln\frac{e\,n_2}{|A_2|} \;,\; |A_3|\ln\frac{e\,n_3}{|A_3|} \Big\} \Big) \;, 
	\end{eqnarray*}
	where $n_1=|V_1|$, $n_2=|V_2|$, $n_3=|V_3|$, $n=\max\{n_1,n_2,n_2\}$ and $\alpha\equiv \max n_i/n_j$.
\end{lemma}

Now, for the choice of $\delta=n^{-5}$, the bounded degree and discrepancy properties in \eqref{eq:boundeddegree}, 
\eqref{eq:discrepancy1}, and \eqref{eq:discrepancy2}
hold for random tripartite hypergraphs. 
This finishes the proof of  Theorem \ref{thm:tcinit}. 

\subsection{Proof of the bounded degree and discrepancy properties in Lemma \ref{lem:discrepancy}}

We first prove the bounded degree properties of \eqref{eq:boundeddegree} hold 
with probability at least $1-\delta$. 
Applying standard concentration inequality, e.g. Bernstein inequality, we get that for some positive constant 
$\delta>0$, 
\begin{eqnarray*}
	\prob\big(\, \deg_1(i) \leq 2pn_2n_3 +\frac83 \log \frac{3n_1}{\delta}\,\big) &\leq& 
		\exp\Big( -\frac{(1/2)(pn_2n_3+(8/3)\log(3n_1/\delta))^2}{(1/3)(pn_2n_3+(8/3)\log(3n_1/\delta)) + n_2n_3p(1-p)}\Big) \\
	&\leq& e^{-\log(3n_1/\delta)}\;,
\end{eqnarray*}
for $n$ sufficiently large, and taking union bound over all choices of $i$, $j$ and $k$, 
$\deg_1(i)$, $\deg_2(j)$, and $\deg_3(k)$'s are uniformly bounded with probability at least $1-\delta/2$.

Similarly, we can apply concentration inequality to bound for some positive constant $\delta>0$
\begin{eqnarray*}
	\prob\Big(\, \deg_{12}(i,j) \leq 2pn_3 + \frac83\log \frac{3n_1n_2}{\delta} \,\Big) &\leq& \exp\Big( -\frac{(1/2)(pn_3+(8/3)\log(3n_1n_2/\delta))^2}{(1/3)(pn_3+(8/3)\log(3n_1n_2/\delta)) + n_3p(1-p)} \Big) \\
		&\leq& e^{-\log(3n_1n_2/\delta)}\;. 
\end{eqnarray*}
Applying the union bound over all choices of $(i,j)$, $(i,k)$ and 
$(j,k)$, we get that the bound holds uniformly 
with probability at least $1-\delta/2$.

Next, we prove that the random hyper graphs satisfy 
the discrepancy properties of \eqref{eq:discrepancy1} and  
\eqref{eq:discrepancy2}. 
For any given subsets $A_1\subseteq[n_1]$, $A_2\subseteq[n_2]$, and $A_2\subseteq[n_3]$, 
let $a_1,a_2$, and $a_3$ denote the cardinality of the subsets, and 
$\be(A_1,A_2,A_3) = p a_1a_2a_3$. 

Let's assume, without loss of generality, that $a_1\leq a_2\leq a_3$. 
We divide the analysis into two cases depending on the size of the smallest subset. 
When at least two of the subsets are large, i.e. $a_2=\Omega(n)$ and $a_3=\Omega(n)$, 
then by bounded degree property, we can prove that \eqref{eq:discrepancy1} holds. 
However, when $a_1$ and $a_2$ are small, e.g. $O(1)$, then the first discrepancy no longer holds, 
and we need a different technique to show concentration.  

\bigskip
\noindent
{\bf Case 1.} When $a_1 \geq n_1/e$. 

From the bounded degree property, we know that 
$\deg_1(i)\leq 2pn_2n_3 + (8/3)\ln(3n_1/\delta)$. 
Then, 
\begin{eqnarray*}
	e(A_1,A_2,A_3) &\leq & a_1(2pn_2n_3 + (8/3)\ln(3n_1/\delta)) \\
		&\leq & a_1(5pn_2n_3 + (8/3)\ln(1/\delta)) \\
		&\leq & 5 a_1 pn_2n_3\,\Big(1 + \frac{\ln(1/\delta) }{pn_2n_3} \Big) \\
		&\leq& 5\,e^2\, \alpha^2 \be(A_1,A_2,A_3)\, \Big(1 + \frac{\ln(1/\delta) }{pn_2n_3} \Big)  \;.
\end{eqnarray*}

\bigskip
\noindent
{\bf Case 2.} When $a_1 < n_1/e$. 

We use the following bound on sum of indicator variables deviating from the mean :
\begin{eqnarray}
	\prob\big(\, e(A_1,A_2,A_3) \geq t \be(A_1,A_2,A_3) \,\big) &\leq& e^{-(1/3)\,\be\,t\ln t}\;,
	\label{eq:concentrationdiscrepancy}
\end{eqnarray}
where we denote $\be(A_1,A_2,A_3)$ by $\be$, which holds for $t\geq 4$. 
For the bound holds with probability at least $1-\delta$, 
we require 
\begin{eqnarray*}
	e^{-(1/3)\be t \ln t}{n_1 \choose a_1}{n_2 \choose a_2}{n_3 \choose a_3} &\leq& \frac{\delta}{n_1n_2n_3}\;,
\end{eqnarray*}
where the term $1/(n_1n_2n_3)$ is chosen to compensate for the union bound over all choices of $a_1$, $a_2$ and $a_3$. Simplifying the combinatorial terms, we get
\begin{eqnarray*}
	e^{-(1/3)\be t \ln t}\Big(\frac{en_1}{a_1}\Big)^{a_1} \Big(\frac{en_2}{a_2}\Big)^{a_2} \Big(\frac{en_3}{a_3}\Big)^{a_3}  e^{\ln(n_1n_2n_3/\delta)} &\leq& 1\;.
\end{eqnarray*}
Equivalently, 
\begin{eqnarray*}
	a_1 \ln \Big(\frac{en_1}{a_1}\Big) + a_2 \ln \Big(\frac{en_2}{a_2}\Big) + a_3 \ln \Big(\frac{en_3}{a_3}\Big) 
		+ \ln\Big(\frac{n_1n_2n_3}{\delta}\Big) &\leq& \frac{\be t \ln t}{3}  \;.
\end{eqnarray*}
We assumed that $a_1\leq n_1/e$, and since $x \ln(n_1/x)$ is monotone in $x\in[1,n_1/e]$, 
we know that $a_1\ln(en_1/a_1) \geq \ln n_1$. 
\begin{eqnarray*}
	4 a_1 \ln \Big(\frac{en_1}{a_1}\Big) + a_2 \ln \Big(\frac{en_2}{a_2}\Big) + a_3 \ln \Big(\frac{en_3}{a_3}\Big) 
		+ \ln\Big(\frac{\alpha^2}{\delta}\Big) &\leq& \frac{\be t \ln t}{3}  \;.
\end{eqnarray*}
To lighten the notations, let's suppose 
$a_1\ln(e\,n_1/a_1) \leq a_2\ln(e\,n_2/a_2) \leq a_3 \ln(e\,n_3/a_3)$. 
Let $t'$ be the smallest number such that $(3/\be)\big(6a_3 \ln(e\,n_3/a_3) +\ln(\alpha^2/\delta)\,\big) = t'\ln t'$. 

For the regime of parameters such that $t'\leq 4$, then $e(A_1,A_2,A_3)\leq 4 \be(A_1,A_2,A_3)$ with probability at least $1-\delta$ the bounded discrepancy condition, in particular the first one, holds. 

For the regime of parameters such that $t'>4$, we can apply \eqref{eq:concentrationdiscrepancy} to get that with probability at least $1-\delta$, the following holds uniformly for all choices of $A_1$, $A_2$, and $A_3$: 
\begin{eqnarray*}
	e(A_1,A_2,A_3) &\leq& t' \be(A_1,A_2,A_3)\;.
\end{eqnarray*}
Since we defined $t'$ to satisfy $\be t'\ln t' = 18 a_3 \ln(e\,n_3/a_3) + 3 \ln(\alpha^2/\delta) $, 
we have 
\begin{eqnarray*}
	e(A_1,A_2,A_3) \ln t' &\leq& 18 a_3 \ln(e\,n_3/a_3) + 3 \ln(\alpha^2/\delta)  \;.
\end{eqnarray*}
As $t'$ upper bounds $e(A_1,A_2,A_3)/\be$, we have 
\begin{eqnarray*}
	e(A_1,A_2,A_3) \ln \Big(\frac{e(A_1,A_2,A_3)}{\be(A_1,A_2,A_3)}\Big) &\leq& 18 a_3 \ln(e\,n_3/a_3) + 3 \ln(\alpha^2/\delta)  \;.
\end{eqnarray*}
\subsection{Proof of Thresholding}\label{app:thresh}
\begin{lemma}\label{lem:threshold}
  Let $\ul, 1\leq \ell\leq r$ be such that $\|\ul-\uol\|_2\leq \alpha$ where $\alpha<1/4$. Also, let $\uol, 1\leq \ell\leq r$ be $\mu$-incoherent unit vectors. Now define $\widetilde{\ul}$ as: 
$$\widetilde{\ul}(i)=\begin{cases}\ul(i)&\text{ if }|\ul(i)|\leq \frac{\mu}{\sqrt{n}},\\
sign(\ul(i))\frac{\mu}{\sqrt{n}}&\text{ if }|\ul(i)|>\frac{\mu}{\sqrt{n}}.\end{cases}$$
Also, let $\widehat{\ul}=\widetilde{\ul}/\|\widetilde{\ul}\|_2$. Then, $\|\widehat{\ul}-\uol\|_2\leq 3\alpha$ $\forall 1\leq \ell\leq r$ and each $\widehat{\ul}$ is $2\mu$-incoherent. 
\end{lemma}
\begin{proof}
As $\|\uol\|_\infty\leq \frac{\mu}{\sqrt{n}}$, hence $\|\widetilde{\ul}-\uol\|_2\leq \|\ul-\uol\|_2\leq \alpha,$ $\forall \ell$. This also implies that $1-\alpha\leq \|\widetilde{\ul}\|_2\leq 1$. Hence, 
$$\|\widehat{\ul}-\uol\|_2\leq \|\widetilde{\ul}-\uol\|_2+\left(\frac{1}{\|\widetilde{\ul}\|_2}-1\right)\leq 3\alpha.$$
Moreover, $\|\widehat{\ul}\|_\infty\leq \frac{\mu}{\sqrt{n}\cdot (1-\alpha)}\leq  \frac{2\mu}{\sqrt{n}}$. Hence proved. 
\end{proof}


%% file: app_altmin.tex
\renewcommand{\ui}{u(i)}
\renewcommand{\uj}{u(j)}
\renewcommand{\uk}{u(k)}

\section{Alternating Minimization Analysis}

\label{sec:r2}

\subsection{Main theorem for rank-two analysis}

\input{rtwocomplete}

\subsection{Technical lemmas for rank-two analysis}

The next lemma shows that all our estimates are $2\mu$-incoherent, 
which in turn allows us to bound the error in the above proof effectively. 
Note that the incoherence of the updates do not increase beyond a global constant ($2\mu$). 
Let $\widehat{\u}_1^{t+1}$ be obtained by update \eqref{eq:r2_up0} and let $\u_1^{t+1}=  \widehat{\u}_1^{t+1}/\|\widehat{\u}_1^{t+1}\|_2$. 
\begin{lemma}\label{lem:r2_incoh}
 Under the hypotheses of Theorem~\ref{thm:r2}, $\u_1^{t+1}$ is $2\mu$-incoherent with probability at least $1-1/n^9$. 
\end{lemma}
\begin{proof}
  Using \eqref{eq:r2_up0} and the definitions of $B$, $C$, $F$, $G$ given in \eqref{eq:bcfg}, we have: 
  \begin{align}
    |\widehat{\u}_1^{t+1}(i)|&\leq \frac{|C_{ii}|}{|B_{ii}|}\frac{\mu}{\sqrt{n}} + \frac{|F_{ii}|}{|B_{ii}|} \frac{\mu}{\sqrt{n}} +  \frac{|G_{ii}|}{|B_{ii}|} \frac{2\mu}{\sqrt{n}}
	\;\leq\;\; \frac{2\mu}{\sqrt{n}}\;,
  \end{align}
where the second inequality follows by bounds on $B_{ii}, C_{ii}, F_{ii}, G_{ii}$ obtained using Lemma~\ref{lem:ip_conc} and the distance bound $d_\infty\left([\u_{1}, \u_{2}], [\uo_1, \uo_2]\right)$. 
\end{proof}

Next, we bound the first error term in \eqref{eq:r2_update}. 
\begin{lemma}\label{lem:r2_err1}
Let $\u=\uo+\du$ and $\v=\vo+\dv$, where $\u, \uo, \v, \vo$ are all unit vectors and $\uo \perp \vo$. Also, let $\|\du\|_2\leq 1$ and $\|\dv\|_2\leq 1$. Then, the following holds:
  $$\|\ip{\u}{\vo}^2\vo-\ip{\u}{\v}^2\v\|\leq 4(\|\du\|_2+\|\dv\|_2)\|\dv\|_2.$$
\end{lemma}
\begin{proof}
Note that,
\begin{equation}
\ip{\u}{\v}^2=(\ip{\u}{\vo}+\ip{\u}{\dv})^2=\ip{\u}{\vo}^2+\ip{\u}{\dv}^2+2\ip{\u}{\vo}\ip{\u}{\dv}. 
  \label{eq:r2_diff1}
\end{equation}
Hence, 
\begin{align}
  \|\ip{\u}{\vo}^2\vo-\ip{\u}{\v}^2\v\|_2&= \|\ip{\u}{\v}^2\dv-\ip{\u}{\dv}^2\vo-2\ip{\u}{\vo}\ip{\u}{\dv}\vo\|_2\nonumber,\\
&\leq \ip{\u}{\v}^2\|\dv\|_2+\|\dv\|^2+2|\ip{\u}{\vo}|\|\dv\|_2.\label{eq:r2_diff2}
\end{align}
Now, $\ip{\u}{\vo}=\ip{\du}{\vo}\leq \|\du\|_2$. Also, $\ip{\u}{\v}^2\leq \ip{\u}{\v}=(\ip{\du}{\vo}+\ip{\uo}{\dv}+\ip{\du}{\dv})\leq 2(\|\du\|+\|\dv\|)$. 
Lemma now follows by combining the above observations with \eqref{eq:r2_diff2}. 
\end{proof}

Now, we bound the third error term in \eqref{eq:r2_update}. Note that although the two individual terms ($(\ip{\u_1}{\uo_2}^2B-F)\uo_2$ and $(\ip{\u_1}{\u_2}^2B-G)\v$) are both small, still it is critical to bound the difference as the individual terms can be as large as a constant, even when $\u_1=\uo_1$ and $\u_2=\uo_2$. However, the difference goes down linearly with $\|\u_2-\uo_2\|_2$. 
\begin{lemma}\label{lem:r2_err4}
Let $B, C, F, G$ be defined as in \eqref{eq:bcfg}. Also, let the assumptions of Theorem~\ref{thm:r2} hold. Also, let $p\geq C \frac{\mu^3\log^2 n}{\gamma^2n^{3/2}}$, where $C>0$ is a global constant. 
Then, the following holds with probability $\geq 1-4/n^9$: 
  $$\|(\ip{\u_1}{\uo_2}^2B-F)\uo_2-(\ip{\u_1}{\u_2}^2B-G)\u_2\|_2\leq 8\gamma \|\u_2-\uo_2\|_2.$$
\end{lemma}
\begin{proof}
Let $\u_2=\uo_2+\du_2$ and $\u_1=\uo_1+\du_1$. Then, 
\begin{equation}\label{eq:r2_diff3}(\ip{\u_1}{\u_2}^2B-G)\u_2=(\ip{\u_1}{\u_2}^2B-G)\uo_2+(\ip{\u_1}{\u_2}^2B-G)\du_2.\end{equation}
Now,
\begin{align}
G_{ii}&=\sum_{jk}\delta_{ijk}\u_1(j)\u_1(k)\u_2(j)\u_2(k)=\sum_{jk}\delta_{ijk} \u_1(j)\u_1(k) (\uo_2(j)+\du_2(j))(\uo_2(k)+\du_2(k))\nonumber\\
&=\sum_{jk}\delta_{ijk}\u_1(j)\u_1(k)\left(\uo_2(j)\uo_2(k)+\du_2(j)\uo_2(k)+\uo_2(j)\du_2(k)+\du_2(j)\du_2(k)\right)=F_{ii}+D^1_{ii}+D^2_{ii}+D^3_{ii}.\label{eq:r2_diff4}
\end{align}
Hence, using  \eqref{eq:r2_diff1}, and \eqref{eq:r2_diff4}, we have: 
  \begin{multline}
    (\ip{\u_1}{\u_2}^2B-G)=(\ip{\u_1}{\uo_2}^2B-F)+(\ip{\u_1}{\uo_2}\ip{\u_1}{\du_2}B-D^1)\\+(\ip{\u_1}{\uo_2}\ip{\u_1}{\du_2}B-D^2)+(\ip{\u_1}{\du_2}^2B-D^3)
  \end{multline}
Combining the above equation with \eqref{eq:r2_diff3}, we get: 
\begin{multline}
  (\ip{\u_1}{\u_2}^2B-G)\u_2-(\ip{\u_1}{\uo_2}^2B-F)\uo_2= (\ip{\u_1}{\uo_2}\ip{\u_1}{\du_2}B-D^1)\uo_2+(\ip{\u_1}{\uo_2}\ip{\u_1}{\du_2}B-D^2)\uo_2\\+(\ip{\u_1}{\du_2}^2B-D^3)\uo_2- (\ip{\u_1}{\u_2}^2B-G)\du_2. 
\end{multline}
Lemma now follows using Lemma~\ref{lem:r2_err2}, \ref{lem:r2_err3}, and the above equation. 
\end{proof}

We now present a few technical lemmas that are critical to our proofs of the above given lemmas. 
\begin{lemma}\label{lem:ip_conc}
  Let $\u, \uo\in \R^n$ be $\mu$-incoherent unit vectors. Also, let $\delta_{jk}, 1\leq j\leq n, 1\leq k\leq n$ be i.i.d. Bernoulli random variables with $\delta_{jk}=1$ w.p. $p\geq C\mu^4 \log^3n/(\gamma^2\cdot n^2)$. 

Then,  the following holds with probability $\geq 1-1/n^{10}$: 
$$|\frac{1}{p}\sum_{jk} \delta_{jk}\u(j)\uo(j) \u(k) \uo(k)-\ip{\u}{\uo}^2|\leq \gamma .$$
where $\gamma\leq C/\log n$, where $C>0$ is a global constant. 
\end{lemma}

\begin{lemma}\label{lem:ip_conc1}
  Let $\u\in \R^n$ be $\mu$-incoherent unit vectors. Also, let $a, b\in \R^n$ be s.t. $|a_i|\leq \frac{\mu}{\sqrt{n}}$ and $\|a\|_2\leq 1$.  Also, let $\delta_{jk}, 1\leq j\leq n, 1\leq k\leq n$ be i.i.d. Bernoulli random variables with $\delta_{jk}=1$ w.p. $p\geq \frac{C\mu^3}{\gamma^2 n^{1.5}}$. 

Then,  the following holds with probability $\geq 1-1/n^{10}$: 
$$|\frac{1}{p}\sum_{jk} \delta_{jk}\u({j})a({j}) \u(k) b(k)-\ip{\u}{a}\ip{\u}{b}|\leq \gamma \|b\|_2.$$
where $\gamma\leq C/\log n$, where $C>0$ is a global constant. 
\end{lemma}

\begin{lemma}\label{lem:r2_err2}
 Let $\u$ be a fixed unit vector and let $a, b, c$ be fixed vectors in $\R^n$. Also, let all $\u, a, b, c \in \R^n$ be s.t. their $L_\infty$ norm is bounded by $\frac{\mu}{\sqrt{n}}$ and $L_2$ norm is bounded by $1$. Also, let $p\geq \frac{C \mu^{3}(\log^2 n)}{\gamma^2\cdot n^{3/2}}$, where $C>0$ is a global constant. 
Then the following holds (w.p. $\geq 1-2/n^{10}$): $$ \|(\ip{\u}{a}\ip{\u}{b}B-R)c\|_2\leq  \gamma\sqrt{1-\ip{\u}{a}^2\ip{\u}{b}^2},$$
where $B, R$ are both diagonal matrices with $B(i,i)=\frac{1}{p}\sum_{jk}\delta_{ijk}(\u(j))^2(\u(k))^2$, and $R(i,i)=\frac{1}{p}\sum_{jk}\delta_{ijk}\u(j)\u(k) a(j) b(k)$. 
\end{lemma}

\begin{lemma}\label{lem:r2_err3}
 Let $\u$ be a fixed unit vector and let $a, b$ be fixed vectors in $\R^n$. Also, let all $\u$ be $\mu$-incoherent unit vectors, and $a$ be such that $\|a\|_\infty\leq \frac{\mu}{\sqrt{n}}$ and $\|a\|_2\leq 1$. 
 Also, let $p\geq \frac{C \mu^{3}(\log^2 n)}{\gamma^2\cdot n^{3/2}}$, where $C>0$ is a global constant. 
Then the following holds (w.p. $\geq 1-2/n^{10}$): $$ \|(\ip{\u}{a}\ip{\u}{b}B-R)\|_2\leq  2\gamma\|b\|_2$$
where $B, R$ are  diagonal matrices s.t. $B(i,i)=\frac{1}{p}\sum_{jk}\delta_{ijk}(\u(j))^2(\u(k))^2$,  $R(i,i)=\frac{1}{p}\sum_{jk}\delta_{ijk}\u(j)\u(k) a(j) b(k)$. 
\end{lemma}

\subsection{Proofs of Technical Lemmas}
\renewcommand{\ui}{u(i)}
\renewcommand{\uj}{u(j)}
\renewcommand{\uk}{u(k)}

\begin{proof}[Proof of Lemma~\ref{lem:ip_conc}]
  Let $X_{jk}=\frac{1}{p} \delta_{jk}\u(j)\uo(j) \u(k) \uo(k).$ Note that, $|X_{jk}|\leq \frac{\mu^4}{p n^2}$. Also, $$\E[\sum_{jk}X_{jk}^2]=\frac{1}{p}\sum_{jk}(\u(j))^2(\uo(j))^2 (\u(k))^2 (\uo(k))^2\leq \frac{\mu^4}{p n^2}.$$
Hence, using Bernstein's inequality, we have: 
$$Pr(|\sum_{jk}X_{jk}-\E[\sum_{jk}X_{jk}]|>t)\leq \exp(-\frac{p n^2}{\mu^4}\cdot \frac{t^2/2}{1+t/3}).$$
Lemma now follows by selecting $t=C/\log n$. 
\end{proof}
\begin{proof}[Proof of Lemma~\ref{lem:ip_conc1}]
  Let $X_{jk}=\frac{1}{p} \delta_{jk}\u(j)a(j) \u(k) b(k).$ Note that, $|X_{jk}|\leq \frac{\mu^3\|b\|_2}{p n^{1.5}}$. Also, $$\E[\sum_{jk}X_{jk}^2]=\frac{1}{p}\sum_{jk}(\u(j))^2(a(j))^2 (\u(k))^2 (b(k))^2\leq \frac{\mu^4\|b\|^2}{p n^2}\leq \frac{\mu^3\|b\|^2}{p n^{1.5}}.$$
Hence, using Bernstein's inequality, we have: 
$$Pr(|\sum_{jk}X_{jk}-\E[\sum_{jk}X_{jk}]|>t)\leq \exp(-\frac{p n^{1.5}}{\mu^3}\cdot \frac{t^2/2}{\|b\|_2^2+\|b\|_2t/3}).$$
Lemma now follows by selecting $t=\gamma \|b\|_2$. 
\end{proof}
\begin{proof}[Proof of Lemma~\ref{lem:r2_err2}]
\begin{align}
(\ip{\u}{a}\ip{\u}{b}B-R)c= \frac{1}{p}\sum_{ijk}\delta_{ijk}c_i (\ip{\u}{a}\ip{\u}{b}(\uj)^2(\uk)^2-\uj\uk a(j) b(k)) \e_i=\sum_{ijk}Z_{ijk},
\end{align}
where $Z_{ijk}= \frac{1}{p}\delta_{ijk}c_i (\ip{\u}{a}\ip{\u}{b}(\uj)^2(\uk)^2-\uj\uk a(j) b(k)) \e_i$. Note that, 
$$\|Z_{ijk}-\E[Z_{ijk}]\|_2\leq \frac{2}{p}c_i \uj\uk\sqrt{1-\ip{\u}{a}^2\ip{\u}{b}^2}\leq \gamma\sqrt{1-\ip{\u}{a}^2\ip{\u}{b}^2}, $$
as $p\geq \frac{C \mu^{3}(\log^2 n)}{\gamma\cdot n^{3/2}}$. 
Also, 
$$\|\sum_{ijk}\E[Z_{ijk}^TZ_{ijk}]\|_2=\| \frac{1}{p}\sum_{ijk}c_i^2(\uj)^2(\uk)^2(\ip{\u}{a}\ip{\u}{b}\uj\uk-a(j)b(k))^2\|_2\leq \frac{1}{p}\frac{\mu^4}{n^2}(1-\ip{\u}{a}^2\ip{\u}{b}^2).$$
Hence, for $p$ and $\gamma$ mentioned above, we have: 
$$\|\sum_{ijk}\E[Z_{ijk}^TZ_{ijk}]\|_2\leq \gamma(1-\ip{\u}{a}^2\ip{\u}{b}^2).$$
Lemma now follows by using Bernstein's inequality and the fact that $\sum_{ijk} Z_{ijk}=0$. 
\end{proof}
\begin{proof}[Proof of Lemma~\ref{lem:r2_err3}]
  Consider the $i$-th element of the diagonal matrix $(\ip{\u}{a}\ip{\u}{b}B-R)= \ip{\u}{a}\ip{\u}{b}B(i,i)-R(i,i)$. Now, using Lemma~\ref{lem:ip_conc}, $|B(i,i)|\leq 1+\gamma$ w.p. $\geq 1-1/n^{10}$. Similarly, using Lemma~\ref{lem:ip_conc1}, $|R(i,i)-\ip{\u}{a}\ip{\u}{b}|\leq \gamma\|b\|_2$. Hence, w.p. $\geq 1-1/n^{10}$, we have: 
$$|\ip{\u}{a}\ip{\u}{b}B(i,i)-R(i,i)|\leq 2\gamma\|b\|_2. $$
Lemma now follows by observing that $\|(\ip{\u}{a}\ip{\u}{b}B-R)\|_2=\max_i |\ip{\u}{a}\ip{\u}{b}B(i,i)-R(i,i)|$ and using the above mentioned bound with union bound. 
\end{proof}

\renewcommand{\ui}{\u_i}
\renewcommand{\uj}{\u_j}
\renewcommand{\uk}{\u_k}

\subsection{Proof of Theorem \ref{thm:rk} and general rank-$r$ analysis of alternating minimization}
\label{app:rkk}

\begin{proof}
We prove the theorem by showing the following for all $q$: 
$$\soq\left((\sdq)^{t+1}+\|\duq^{t+1}\|_2\right) \;\leq\; \frac{1}{2}d_\infty([U, \Sigma], [U^*, \Sigma^*]).$$
The update for $\widehat{\u_q}^{t+1}$ is given by: 
\begin{multline}
  \label{eq:rk_k_update}
  \widehat{\uq}^{t+1}(i)=\frac{\sum_{jk}\delta_{ijk}\soq\cdot\uqj\uqk\uoqj\uoqk}{\sum_{jk}\delta_{ijk}\uqj^2\uqk^2}\uoqi\\+\frac{\sum_{\ell\neq q}\sum_{jk}\delta_{ijk}\uqj\uqk(\sol\cdot\uoli\uolj\uolk-\sll\cdot\uli\ulj\ulk)}{\sum_{jk}\delta_{ijk}\uqj^2\uqk^2}.
\end{multline}
It can be written as a vector update, 
\begin{align}
  \widehat{\uq}^{t+1} = \soq\ip{\uq}{\uoq}^2 \uoq - B^{-1}(\soq\ip{\uq}{\uoq}^2 B - \soq C)\uoq+\sum_{\ell\neq q}\left(\sol\ip{\uq}{\uol}^2\uol-\sll\ip{\uq}{\ul}^2\ul\right)  \nonumber\\+ \sum_{\ell\neq q}B^{-1}\left(\sol\cdot (\ip{\uq}{\uol}^2B-F_\ell)\uol-\sll\cdot (\ip{\uq}{\ul}^2B-G_\ell)\ul\right),
\end{align}
where $B$, $C$, $F_\ell$, $G_\ell$ are all diagonal matrices, s.t.,
\begin{align}
  B(i,i)&=\sum_{jk}\delta_{jk}\uqj^2\uqk^2, \text{  } C(i,i)=\sum_{jk}\delta_{jk}\uqj\uoqj\uqk\uoqk, \nonumber\\
F_{\ell}(i,i)&=\sum_{jk}\delta_{ijk}\uqj\uqk\uolj\uolk, \text{ and } G_{\ell}(i,i)=\sum_{jk}\delta_{ijk}\uqj\uqk\ulj\ulk. \label{eq:rk_bcfg}
\end{align}
We decompose the error terms $ \widehat{\uq}^{t+1} -\soq \uoq = \err_q^0 + \sum_{\ell\neq q} (\err_\ell^1 + \err_\ell^2)$ 
and provide upper bounds for each, where 
\begin{align}
\err^0_q&\;\equiv \;\soq(\ip{\uq}{\uoq}^2-1)\uoq-\soq B^{-1}(\ip{\uq}{\uoq}^2 B - C)\uoq,\nonumber\\
\err^1_\ell&\;\equiv\;\sol\ip{\uq}{\uol}^2\uol-\sll\ip{\uq}{\ul}^2\ul,\nonumber\\
\err^2_\ell&\;\equiv\;B^{-1}\left(\sol\cdot (\ip{\uq}{\uol}^2B-F_\ell)\uol-\sll\cdot (\ip{\uq}{\ul}^2B-G_\ell)\ul\right).
\end{align}
Using Lemma~\ref{lem:r2_err2}, we have for all $p$ satisfying $p\geq (C \mu^3 (\log n)^2)/(\gamma^2\,n^{3/2})$, with probability at least $1-2/n^{10}$: 
\begin{equation}
  \label{eq:rk_err0}
  \|\err^0_q\|_2 \;\leq\; \soq\left(\sqrt{1-\ip{\uq}{\uoq}^2}+2\gamma \right)\sqrt{1-\ip{\uq}{\uoq}^2}\leq \soq\left(\|\duq\|_2+2\gamma\right)\|\duq\|_2. 
\end{equation}
Eventually, we set $\gamma\leq \frac{1}{1600 r}\cdot \frac{\sigma^*_{min}}{\sigma^*_{max}}$ to prove the theorem.
Using Lemma~\ref{lem:rk_err1}, we have (w.p. $\geq 1-1/n^8$): 
\begin{equation}
  \label{eq:rk_err1}
  \sum_{\ell\neq q}\|\err^1_\ell\|_2 \;\leq\; 8\sum_{\ell\neq q}(\|\duq\|_2+\|\dul\|_2)\cdot \sol\cdot (\|\dul\|_2+\sdl).
\end{equation}
Using Lemma~\ref{lem:rk_err2}, we get (w.p. $\geq 1-1/n^8$): 
\begin{equation}
  \label{eq:rk_err2}
  \sum_{\ell\neq q}\|\err^2_\ell\|_2 \;\leq\; 16\gamma \sum_{\ell\neq q}\sol\cdot (\sdl+\|\dul\|_2).
\end{equation}
Using \eqref{eq:rk_k_update}, \eqref{eq:rk_err0}, \eqref{eq:rk_err1}, \eqref{eq:rk_err2}, we have (w.p. $\geq 1-3/n^8$): 
\begin{equation}
  \label{eq:rk_up1}
  \widehat{\uq}^{t+1} = \sq^{t+1}\uq^{t+1}=\soq \uoq+\err_q,
\end{equation}
where, 
\begin{equation}
\|\err_q\|_2 \;\leq\; \soq\left(\|\dul\|_2+2\gamma\right)\|\duq\|_2+8\sum_{\ell\neq q}(\|\duq\|_2+\|\dul\|_2 + 2\gamma)\, \sol\, (\|\dul\|_2+\sdl)\;.
\end{equation}
Now, since $\|\dul\|_2\leq \frac{1}{1600 r}\cdot \frac{\sigma^*_{min}}{\sigma^*_{max}}, \forall \ell$, and $\gamma\leq \frac{1}{1600 r}$, we have (w.p. $\geq 1-3/n^8$): 
\begin{equation}
  \label{eq:rk_up2}
  \|\err_q\|_2 \;\leq\; \frac{\soq}{16}\cdot \frac{\sigma^*_{min}}{\sigma^*_{max}}\|\duq\|_2+\frac{1}{16}\cdot \sigma^*_{min}\cdot d_\infty([U,\ \Sigma], [U^*,\ \Sigma^*]), 
\end{equation}
Using \eqref{eq:rk_up1} and \eqref{eq:rk_up2}, 
and the fact that   $|\sq^{t+1}-\soq| \leq | \sq^{t+1}\uq^{t+1}-\uoq\soq| $ for normalized vectors $\uq^{t+1}$ and $\uoq$,  
we have: 
\begin{equation}
  \label{eq:rk_up3}
  |\sq^{t+1}-\soq| \;\leq\;  \frac{\soq}{16}\|\duq\|_2+\frac{\sigma_{min}^*}{16}d_\infty([U,\ \Sigma], [U^*,\ \Sigma^*])\leq \frac{\soq}{8} d_\infty([U,\ \Sigma], [U^*,\ \Sigma^*]). 
\end{equation}
Similarly, using \eqref{eq:rk_up1} and \eqref{eq:rk_up3}, we have: 
\begin{equation}
  \label{eq:rk_up4}
  \soq\|\uq^{t+1}-\uoq\|_2 \;\leq\;  \frac{\soq}{4}d_\infty([U,\ \Sigma], [U^*,\ \Sigma^*]). 
\end{equation}
That is, 
\begin{equation}
  \label{eq:rk_up5}
  (\sdq)^{t+1}+\|\duq^{t+1}\|_2 \;\leq\; \frac{1}{2}d_\infty([U,\ \Sigma], [U^*,\ \Sigma^*]). 
\end{equation}
First part of the Theorem now follows by observing that $d_\infty([U^{t+1},\ \Sigma^{t+1}], [U^*,\ \Sigma^*])=\max_q \soq\left((\sdq)^{t+1}+\|\duq^{t+1}\|_2\right)$ and by using the above equation. 

Second part of the Theorem follows directly from Lemma~\ref{lem:rk_incoh}. 


\end{proof}

\subsection{Technical lemmas for general rank-$r$ analysis}

\begin{lemma}\label{lem:rk_incoh}
Let $\hat{\u}_q^{t+1}$ be obtained by update \eqref{eq:rk_k_update} and let $\u_q^{t+1}=  \hat{\u}_q^{t+1}/\|\hat{\u}_q^{t+1}\|_2$. Also, let the conditions given in Theorem~\ref{thm:rk} hold. Then, w.p. $\geq 1-1/n^9$, $\u_q^{t+1}$ is $2\mu$-incoherent. 
\end{lemma}
\begin{proof}
  Using \eqref{eq:rk_k_update} and the definitions of $B$, $C$, $F_\ell$, $G_\ell$ given in \eqref{eq:rk_bcfg}, we have: 
  \begin{align}
    |\hat{\u}_q^{t+1}(i)|&\leq \soq\frac{|C(i,i)|}{|B(i,i)|}\frac{\mu}{\sqrt{n}} + \sum_\ell \sol\frac{|F_{\ell}(i,i)|}{|B(i,i)|} \frac{\mu}{\sqrt{n}} +  \sll\frac{|G_{\ell}(i,i)|}{|B(i,i)|} \frac{\mu}{\sqrt{n}},\nonumber\\
&\leq\left(\soq(1+\gamma)/(1-\gamma)+\sum_\ell \sol (\gamma+\|\dul\|_2) + 2\sum_\ell \sol\cdot (1+\sdl)\cdot (\gamma+\|\dul\|_2)\right)\mu/\sqrt{n},\nonumber\\
&\leq \soq (1+\frac{1}{100}) \cdot \frac{\mu}{\sqrt{n}}\label{eq:rk_incoh1}
  \end{align}
where the second inequality follows by bounds on $B_{ii}, C_{ii}, F_{ii}, G_{ii}$ obtained using Lemma~\ref{lem:ip_conc1} and the distance bound $d_\infty\left([\u_{1}, \u_{2}], [\uo_1, \uo_2]\right)$. 

Lemma now follows using \eqref{eq:rk_incoh1} and the bound on $|\sq^{t+1}-\soq|$ given by \eqref{eq:rk_up5}. 
\end{proof}

\begin{lemma}\label{lem:rk_err1}
Let $\dul=\uol-\ul$ and $\sdl = |\sll-\sol|/ \sol$, where $\|\dul\|_2\leq 1$. Let $\uol, \ul$, $\forall 1\leq \ell \leq r$ be unit vectors and let $\ip{\uol}{\uoq}=0$, $\forall \ell\neq q$. Then, the following holds:
  $$\|\sol\cdot\ip{\uq}{\uol}^2\uol-\sll\cdot\ip{\uq}{\ul}^2\ul\|\leq 4\sol(\|\dul\|_2+\|\duq\|_2)(\|\dul\|_2+\sdl).$$
\end{lemma}
\begin{proof}
Let $\sll=\sol+\Delta^{\sigma}_\ell$. 

Now, 
\begin{equation}
\ip{\uq}{\ul}^2=(\ip{\uq}{\uol}+\ip{\uq}{\dul})^2=\ip{\uq}{\uol}^2+\ip{\uq}{\dul}^2+2\ip{\uq}{\uol}\ip{\uq}{\dul}. 
  \label{eq:rk_diff1}
\end{equation}
Hence, 
\begin{align}
  \|\ip{\uq}{\uol}^2\uol-\ip{\uq}{\ul}^2\ul\|_2&= \|\ip{\uq}{\ul}^2\dul-\ip{\uq}{\dul}^2\uol-2\ip{\uq}{\uol}\ip{\uq}{\dul}\uol\|_2\nonumber,\\
&\leq \ip{\uq}{\ul}^2\|\dul\|_2+\|\dul\|^2+2|\ip{\uq}{\uol}|\|\dul\|_2.\label{eq:rk_diff2}
\end{align}
Now, $\ip{\uq}{\uol}=\ip{\duq}{\uol}\leq \|\duq\|_2$. Also, $\ip{\uq}{\ul}^2\leq \ip{\uq}{\ul}=(\ip{\duq}{\uol}+\ip{\uoq}{\dul}+\ip{\duq}{\dul})\leq 2(\|\duq\|+\|\dul\|)$. 
Using the above observations with \eqref{eq:rk_diff2}, we have: 
$$\|\sol\cdot\ip{\uq}{\uol}^2\uol-\sll\cdot\ip{\uq}{\ul}^2\ul\|\leq \sol\|\ip{\uq}{\uol}^2\uol-\ip{\uq}{\ul}^2\ul\|_2+\sol\cdot \sdl\ip{\uq}{\ul}^2. $$
Lemma now follows by combining the above equation with the above given bound on $\ip{\uq}{\ul}^2$. 
\end{proof}

\begin{lemma}\label{lem:rk_err2}
Let $\ul$, $\dul,\ \sdl,\ \forall \ell$ be as defined in Theorem~\ref{thm:rk} and let $B, F_\ell, G_\ell$ be as defined in \eqref{eq:rk_bcfg}. Also, let $T$ and $p$ satisfy assumptions of Theorem~\ref{thm:rk}. Then, the following holds with probability $\geq 1-4/n^9$: 
  $$\|\sol\cdot(\ip{\uq}{\uol}^2B-F_\ell)\uol-\sll\cdot(\ip{\uq}{\ul}^2B-G_\ell)\ul\|_2\leq 8\gamma \cdot \sol\cdot (\sdl +\|\dul\|_2).$$
\end{lemma}
\begin{proof}
\begin{equation}\label{eq:rk_diff3}(\ip{\uq}{\ul}^2B-G_\ell)\ul=(\ip{\uq}{\ul}^2B-G_\ell)\uol+(\ip{\uq}{\ul}^2B-G_\ell)\dul.\end{equation}
Now,
\begin{align}
G_\ell(i,i)&=\sum_{jk}\delta_{ijk}\uqj\uqk\ulj\ulk=\sum_{jk}\delta_{ijk} \uqj\uqk (\uolj+\dulj)(\uolk+\dulk)\nonumber\\
&=\sum_{jk}\delta_{ijk}\uqj\uqk(\uolj\uolk+\dulj\uolk+\uolj\dulk+\dulj\dulk)\nonumber\\&=F_\ell(i,i)+D^1(i,i)+D^2(i,i)+D^3(i,i).\label{eq:rk_diff4}
\end{align}
Using  \eqref{eq:rk_diff1}, and \eqref{eq:rk_diff4}, we have: 
  \begin{multline}
    (\ip{\uq}{\ul}^2B-G_\ell)=(\ip{\uq}{\uol}^2B-F_\ell)+(\ip{\uq}{\uol}\ip{\uq}{\dul}B-D^1)\\+(\ip{\uq}{\uol}\ip{\uq}{\dul}B-D^2)+(\ip{\uq}{\dul}^2B-D^3)
  \end{multline}
Combining the above equation with \eqref{eq:rk_diff3}, we get: 
\begin{multline}
  (\ip{\uq}{\ul}^2B-G_\ell)\ul-(\ip{\uq}{\uol}^2B-F_\ell)\uol=(\ip{\uq}{\uol}\ip{\uq}{\dul}B-D^1)\uol+(\ip{\uq}{\uol}\ip{\uq}{\dul}B-D^2)\uol\\+(\ip{\uq}{\dul}^2B-D^3)\uol - (\ip{\uq}{\ul}^2B-G_\ell)\dul. 
\end{multline}

Hence, using Lemma~\ref{lem:r2_err2} and \ref{lem:r2_err3}, we get: 
\begin{equation}
  \sol\|(\ip{\uq}{\ul}^2B-G_\ell)\ul-(\ip{\uq}{\uol}^2B-F_\ell)\uol\|_2\leq 8\gamma \sol\|\dul\|_2.\label{eq:rk_diff5}
\end{equation}

Similarly, using Lemma~\ref{lem:r2_err2}, we have: 
\begin{equation}
  \label{eq:rk_diff6}
  \sdl\cdot \sol \|(\ip{\uq}{\ul}^2B-G_\ell)\ul\|_2\leq \gamma \cdot \sdl \cdot \sol. 
\end{equation}
Lemma now follows by combining \eqref{eq:rk_diff5}, \eqref{eq:rk_diff6}, and by using triangular inequality. 
\end{proof}

\subsection{Proof of Lemma \ref{lem:tensordist}}
\label{app:tensordist}
\begin{proof}
	\begin{align*}
&	\| \sigma_a(u_a\otimes u_a \otimes u_a)-\sigma^*_a(u_a^*\otimes u_a^* \otimes u_a^*) \|_F \\
& \hspace{1cm} \leq \; \tepsilon\sigma_a^*+ \sigma_a^* \|(u_a\otimes u_a \otimes u_a)-(u_a^*\otimes u_a^* \otimes u_a^*)\|_F\\
& \hspace{1cm} \leq \; \tepsilon\sigma_a^*+ \sigma_a^*
		\Big( \|(u_a-u_a^*)\otimes u_a^* \otimes u_a^*\|_F+\|u_a\otimes (u_a-u_a^*) \otimes u_a^*\|_F - \|u_a\otimes u_a \otimes (u_a-u^*_a)\|_F\Big)\\
& \hspace{1cm} \leq \; 4\,\tepsilon\,\sigma_a^* \;.
	\end{align*}
Similarly, applying Cauchy-Schwartz,we get for $a\neq b$,  
	\begin{align*}
&	\langle \sigma_a(u_a\otimes u_a \otimes u_a)-\sigma^*_a(u_a^*\otimes u_a^* \otimes u_a^*), 	
		\sigma_b(u_b\otimes u_b \otimes u_b)-\sigma^*_b(u_b^*\otimes u_b^* \otimes u_b^*) \rangle \\
& \hspace{1cm} \leq \; 16\,\tepsilon^2\,\sigma_a^*\sigma_b^*\;.
	\end{align*}
It follows that 
	\begin{eqnarray*}
	\|T-\hT\|_F^2 &=& \sum_{a,b\in[r]} \langle \sigma_a(u_a\otimes u_a \otimes u_a)-\sigma^*_a(u_a^*\otimes u_a^* \otimes u_a^*) ,
		\sigma_b(u_b\otimes u_b \otimes u_b)-\sigma^*_b(u_b^*\otimes u_b^* \otimes u_b^*) \rangle\\
	&\leq& 16 \,\tepsilon^2\, (\sum_a\sigma^*_a)^2 \;\; \leq\; (4 \,\tepsilon\, r\,\sigma_{\rm max}^*)^2 \;\leq\;\; 
	(4 \,\tepsilon\, r^{1/2}\,\|T\|_F)^2\;.
	\end{eqnarray*}
\end{proof}

%% file: rtwocomplete.tex

In this section, we provide convergence analysis for Algorithm~\ref{algo:rkk} for 
the special case of a rank-$2$ orthonormal tensor $T$ with equal singular values, 
i.e. $T=\uo_1\otimes \uo_1\otimes \uo_1+\uo_2\otimes \uo_2\otimes \uo_2$, 
where $\uo_1, \uo_2 \in \R^n$ are orthonormal vectors satisfying $\mu$-incoherence, 
i.e., $\|\uoi\|_\infty\leq {\mu}/{\sqrt{n}}$. 
The purpose of this example is to highlight the proof ideas and we fix $\sigma_1, \sigma_2$ 
to be both one at each step of Algorithm~\ref{algo:rkk} for simplicity. 
The following theorem proves the desired linear convergence. 
  Let $[\u^t_1, \u^t_2]$ denote the current estimate at the $t$-th iteration of 
  Algorithm~\ref{algo:rkk}. For brevity, we drop the superscript indexing time 
  and let $[\u_1, \u_2]$ denote $[\u^t_1, \u^t_2]$ whenever it is clear from the context.

\begin{theorem}\label{thm:r2}
  If $\u_1$ and $\u_2$ are $2\mu$-incoherent, then 
  there exists a positive constant $C$ such that for $p\geq C\frac{\mu^3\log^2 n}{n^{1.5}}$ 
  the following holds (w.p. $\geq 1-\log(1/\epsilon)/n^8$): 
$$d_\infty\left([\u_1^{t+1}, \u_2^{t+1}], [\uo_1, \uo_2]\right)\leq \frac{1}{4}d_\infty\left([\u_{1}, \u_{2}], [\uo_1, \uo_2]\right),$$
where $d_\infty([\u_{1}, \u_{2}], [\uo_1, \uo_2])=\max_{i, 1\leq i\leq 2}\|\u_i-\uoi\|_2.$ 
Moreover, $\u_1^{t+1}$, $\u_2^{t+1}$ are both $2\mu$-incoherent. 
\end{theorem}
\begin{proof}
We claim that with probability at least $1-1/n^{8}$,  
\begin{eqnarray*}
  \|\u^{t+1}_i-\uo_i\|_2 \;\leq\; \frac{1}{4}\cdot d_{\infty}\left([\u_1, \u_2], [\uo_1, \uo_2]\right)\;, 
\end{eqnarray*}
for both $i\in\{1,2\}$. This proves the desired bound. 
Incoherence of $[\u^{t+1}_1, \u^{t+1}_2]$ follows from Lemma~\ref{lem:r2_incoh}. 
Without loss of generality, we only prove the claim for $i=1$. 
Recall that $\widehat{\u}^{t+1}_1$ is 
the solution of the least squares problem 
in Step 11 of Algorithm~\ref{algo:rkk}, and can be written as 
\begin{equation}\label{eq:r2_up0}
\widehat{\u}_1^{t+1}(i)=\frac{\sum_{jk}\delta_{ijk}\u_1(j)\u_1(k)\uo_1(j)\uo_1(k)}{\sum_{jk}\delta_{ijk}(\u_1(j))^2(\u_1(k))^2}\uo_1(i)+\frac{\sum_{jk}\delta_{ijk}\u_1(j)\u_1(k)(\uo_2(i)\uo_2(j)\uo_2(k)-\u_2(i)\u_2(j)\u_2(k))}{\sum_{jk}\delta_{ijk}(\u_1(j))^2(\u_1(k))^2}.\end{equation}
Note that the update that can be written in a vector form: 
\begin{align}
  \widehat{\u}^{t+1} = \ip{\u_1}{\uo_1}^2 \uo_1 + \ip{\u_1}{\uo_2}^2\uo_2-\ip{\u_1}{\u_2}^2\u_2 + B^{-1}(\ip{\u_1}{\uo_1}^2 B - C)\uo_1 
  \nonumber\\
  + B^{-1}(\ip{\u_1}{\uo_2}^2B-F)\uo_2-B^{-1}(\ip{\u_1}{\u_2}^2B-G)\u_2,\label{eq:r2_update}
\end{align}
where $B$, $C$, $F, G$ are all diagonal matrices, s.t., 
\begin{align}B_{ii}&=\sum_{jk}\delta_{ijk}(\u_1(j))^2(\u_1(k))^2,\ \ \ C_{ii}=\sum_{jk}\delta_{ijk}\u_1(j)\u_1(k)\uo_1(j)\uo_1(k),\nonumber\\
F_{ii}&=\sum_{jk}\delta_{ijk}\u_1(j)\u_1(k)\uo_2(j)\uo_2(k),\ \ \ G_{ii}=\sum_{jk}\delta_{ijk}\u_1(j)\u_1(k)\u_2(j)\u_2(k)\;.
\label{eq:bcfg}
\end{align}
Let $\widehat{\u}_1^{t+1} - \langle \u_1,\uo_1\rangle^2\uo_1 = \err^0 + \err^1 + \err^2$, such that 
\begin{align}\err^0&=\ip{\u_1}{\uo_2}^2\uo_2-\ip{\u_1}{\u_2}^2\u_2,\nonumber\\\err^1&= B^{-1}(\ip{\u_1}{\uo_1}^2 B - C)\uo,\nonumber\\
\err^2&=B^{-1}(\ip{\u_1}{\uo_2}^2B-F)\uo_2-B^{-1}(\ip{\u_1}{\u_2}^2B-G)\u_2 \;. \end{align}
We separate the analysis for each of the error terms. 
Using Lemma~\ref{lem:r2_err1}, we have: 
\begin{equation}
  \label{eq:r2_err0}
  \|\err^0\|_2\leq 4d_\infty\left([\u_{1}, \u_{2}], [\uo_1, \uo_2]\right)\|\u_2-\uo_2\|_2. 
\end{equation}
Setting $p\geq C \frac{\mu^3\log^2 n}{\gamma^2n^{3/2}}$ for a $\gamma$ 
to be chosen appropriately later 
and using Lemma~\ref{lem:r2_err2} and Lemma~\ref{lem:ip_conc}, we have (w.p. $\geq 1-2/n^9$): 
\begin{equation}
  \label{eq:r2_err1}
  \|\err^1\|_2\leq \frac{\gamma}{1-\gamma} \|\u_1-\uo_1\|_2. 
\end{equation}
Similarly, using Lemma~\ref{lem:r2_err4} and $p\geq C \frac{\mu^3\log^2 n}{\gamma^2n^{3/2}}$, we have (w.p. $\geq 1-1/n^{9}$): 
\begin{equation}
  \label{eq:r2_err2}
  \|\err^2\|\leq 8\frac{\gamma}{1-\gamma} \cdot \|\u_2-\uo_2\|_2. 
\end{equation}
We want to upper bound the error: 
\begin{eqnarray*}
\|\widehat{\u}_1^{t+1}-\uo_1\|_2 &\leq&  \|\widehat{\u}_1^{t+1}-\langle \u_1,\uo_1\rangle^2\uo_1\|_2 + \|(\langle \u_1,\uo_1\rangle^2 -1)\uo_1\|_2 \;.
\end{eqnarray*}
Since $1-\langle \u_1,\uo_1\rangle^2 = (1/2) \|\u_1 -\uo_1\|_2^2$, 
we have from \eqref{eq:r2_err0}, \eqref{eq:r2_err1}, and \eqref{eq:r2_err2}
that (w.p. $\geq 1-10/n^{9}$): 
\begin{eqnarray*}
	\|\widehat{\u}_1^{t+1}-\uo_1\|_2 &\leq& 
	\Big(\frac{\gamma}{1-\gamma} +\frac{\|\u_1-\uo_1\|_2 }{2} \Big) \|\u_1-\uo_1\|_2 + 
	\left(8\frac{\gamma}{1-\gamma} +4d_\infty\left([\u_{1}, \u_{2}], [\uo_1, \uo_2]\right)\right) \|\u_2-\uo_2\|_2\;.
\end{eqnarray*}
Setting $\gamma\leq 1/200$ and for 
$d_\infty\left([\u_{1}, \u_{2}], [\uo_1, \uo_2]\right)\leq {1}/{200}$
as per our assumption, 
this proves the desired bound. 
\end{proof}


%% file: tensorcomp.bbl
\begin{thebibliography}{10}

\bibitem{AGHKT12}
Anandkumar Anima, Ge~Rong, Hsu Daniel, M.~Kakade Sham, and Matus Telgarsky.
\newblock Tensor decompositions for learning latent variable models.
\newblock {\em CoRR}, abs/1210.7559, 2012.

\bibitem{AGJ14}
A.~Anandkumar, R.~Ge, and M.~Janzamin.
\newblock Guaranteed non-orthogonal tensor decomposition via alternating rank-1
  updates.
\newblock {\em arXiv preprint arXiv:1402.5180}, 2014.

\bibitem{SL08}
V.~De~Silva and L.-H. Lim.
\newblock Tensor rank and the ill-posedness of the best low-rank approximation
  problem.
\newblock {\em SIAM Journal on Matrix Analysis and Applications},
  30(3):1084--1127, 2008.

\bibitem{JNS13}
P.~Jain, P.~Netrapalli, and S.~Sanghavi.
\newblock Low-rank matrix completion using alternating minimization.
\newblock In {\em STOC}, pages 665--674, 2013.

\bibitem{hardt2013provable}
Moritz Hardt.
\newblock On the provable convergence of alternating minimization for matrix
  completion.
\newblock {\em arXiv preprint arXiv:1312.0925}, 2013.

\bibitem{CT10}
E.~J. Cand{\`e}s and T.~Tao.
\newblock The power of convex relaxation: Near-optimal matrix completion.
\newblock {\em Information Theory, IEEE Transactions on}, 56(5):2053--2080,
  2010.

\bibitem{Hit27}
F.~L. Hitchcock.
\newblock {\em The expression of a tensor or a polyadic as a sum of products}.
\newblock sn., 1927.

\bibitem{CC70}
J.~D. Carroll and J.~Chang.
\newblock Analysis of individual differences in multidimensional scaling via an
  n-way generalization of Òeckart-youngÓ decomposition.
\newblock {\em Psychometrika}, 35(3):283--319, 1970.

\bibitem{Har70}
R.~A. Harshman.
\newblock Foundations of the parafac procedure: Models and conditions for an"
  explanatory" multimodal factor analysis.
\newblock 1970.

\bibitem{ZG01}
T.~Zhang and G.~H. Golub.
\newblock Rank-one approximation to high order tensors.
\newblock {\em SIAM Journal on Matrix Analysis and Applications},
  23(2):534--550, 2001.

\bibitem{ADKM11}
E.~Acar, D.~M. Dunlavy, T.~G. Kolda, and M.~M{\o}rup.
\newblock Scalable tensor factorizations for incomplete data.
\newblock {\em Chemometrics and Intelligent Laboratory Systems}, 106(1):41--56,
  2011.

\bibitem{TB05}
G.~Tomasi and R.~Bro.
\newblock Parafac and missing values.
\newblock {\em Chemometrics and Intelligent Laboratory Systems},
  75(2):163--180, 2005.

\bibitem{Bro98}
R.~Bro.
\newblock {\em Multi-way analysis in the food industry: models, algorithms, and
  applications}.
\newblock PhD thesis, 1998.

\bibitem{WM01}
B~Walczak and DL~Massart.
\newblock Dealing with missing data: Part i.
\newblock {\em Chemometrics and Intelligent Laboratory Systems}, 58(1):15--27,
  2001.

\bibitem{LMP13}
J.~Liu, P.~Musialski, P.~Wonka, and J.~Ye.
\newblock Tensor completion for estimating missing values in visual data.
\newblock {\em Pattern Analysis and Machine Intelligence, IEEE Transactions
  on}, 35(1):208--220, 2013.

\bibitem{MHWG13}
C.~Mu, B.~Huang, J.~Wright, and D.~Goldfarb.
\newblock Square deal: Lower bounds and improved relaxations for tensor
  recovery.
\newblock {\em arXiv preprint arXiv:1307.5870}, 2013.

\bibitem{CR09}
E.~J. Cand{\`e}s and B.~Recht.
\newblock Exact matrix completion via convex optimization.
\newblock {\em Foundations of Computational mathematics}, 9(6):717--772, 2009.

\bibitem{KMO10}
R.~H. Keshavan, A.~Montanari, and S.~Oh.
\newblock Matrix completion from a few entries.
\newblock {\em IEEE Trans. Inform. Theory}, 56(6):2980--2998, June 2010.

\bibitem{TomiokaS13}
Ryota Tomioka and Taiji Suzuki.
\newblock Convex tensor decomposition via structured schatten norm
  regularization.
\newblock In {\em NIPS}, pages 1331--1339, 2013.

\bibitem{AFK01}
Y.~Azar, A.~Fiat, A.~Karlin, F.~McSherry, and J.~Saia.
\newblock Spectral analysis of data.
\newblock In {\em Proceedings of the thirty-third annual ACM symposium on
  Theory of computing}, pages 619--626. ACM, 2001.

\bibitem{FKS89}
J.~Friedman, J.~Kahn, and E.~Szemer{\'e}di.
\newblock On the second eigenvalue in random regular graphs.
\newblock In {\em Proceedings of the Twenty-First Annual ACM Symposium on
  Theory of Computing}, pages 587--598, Seattle, Washington, USA, may 1989.
  ACM.

\bibitem{FeO05}
U.~Feige and E.~Ofek.
\newblock Spectral techniques applied to sparse random graphs.
\newblock {\em Random Struct. Algorithms}, 27(2):251--275, 2005.

\bibitem{BO09}
R.~Berke and M.~Onsj{\"o}.
\newblock Propagation connectivity of random hypergraphs.
\newblock In {\em Stochastic Algorithms: Foundations and Applications}, pages
  117--126. Springer, 2009.

\end{thebibliography}
